\documentclass[runningheads]{llncs}
\usepackage[T1]{fontenc}
\usepackage{graphicx}
\usepackage{subfigure}
\usepackage{booktabs} 
\usepackage{multirow}
\usepackage{amsmath}
\usepackage{amssymb}

\begin{document}
\title{Enhanced Distribution Modelling via Augmented Architectures For Neural ODE Flows}
%
%
\author{Etrit Haxholli, Marco Lorenzi}
\authorrunning{Etrit Haxholli, Marco Lorenzi}
%
\institute{Inria
\email{}\\
\url{etrit.haxholli@inria.fr}}
\maketitle              

\begin{abstract}
While the neural ODE formulation of normalizing flows such as in FFJORD enables us to calculate the determinants of free form Jacobians in $\mathcal{O}(D)$ time, the flexibility of the transformation underlying neural ODEs has been shown to be suboptimal. In this paper, we present AFFJORD, a neural ODE-based normalizing flow which enhances the representation power of FFJORD by defining the neural ODE through special augmented transformation dynamics which preserve the topology of the space. Furthermore, we derive the Jacobian determinant of the general augmented form by generalizing the chain rule in the continuous sense into the \emph{cable rule}, which expresses the forward sensitivity of ODEs with respect to their initial conditions. The cable rule gives an explicit expression for the Jacobian of a neural ODE transformation, and provides an elegant proof of the instantaneous change of variable. Our experimental results on density estimation in synthetic and high dimensional data, such as MNIST, CIFAR-10 and CelebA ($32\times32$), show that AFFJORD outperforms the baseline FFJORD through the improved flexibility of the underlying vector field.  
\end{abstract}
\section{Introduction}
Normalizing flows are diffeomorphic random variable transformations providing a powerful theoretical framework for generative modeling and probability density estimation \cite{rezende_NF}.
While the practical application of normalizing flows is generally challenging due to computational bottlenecks, most notably regarding the $\mathcal{O}(D^3)$ computation cost of the Jacobian determinant,  different architectures have been proposed in order to scale normalizing flows to high dimensions while at the same time ensuring the flexibility and bijectivity of the transformations \cite{rezende_NF,real_nvp,Kobyzev_2021}. The common strategy consists of placing different architectural restrictions on the model, to enforce special Jacobian forms, with less computationally demanding determinants. 


A noteworthy approach is based on neural ODEs \cite{node_chen2018}, as they enable us to calculate the determinants of free form Jacobians in $\mathcal{O}(D)$ time \cite{ffjord_chen2019}. 
More specifically, the rule for the instantaneous change of variable \cite{node_chen2018} provides an important theoretical contribution to normalizing flows, as it yields a closed form expression of the Jacobian determinant of a neural ODE transformation.

In this case, calculating the Jacobian determinant simplifies to calculating the integral of the divergence of the vector field along the transformation trajectory. Such models are known as Continuous Normalizing Flows (CNFs). In \cite{ffjord_chen2019}, these ideas are further explored and computational simplifications are introduced, notably the use of Hutchinson’s trace estimator \cite{hutchinson_trick_est}. The resulting model is named FFJORD.

\begin{figure}
\centering
\begin{tabular}{ll}
\includegraphics[scale=0.4]{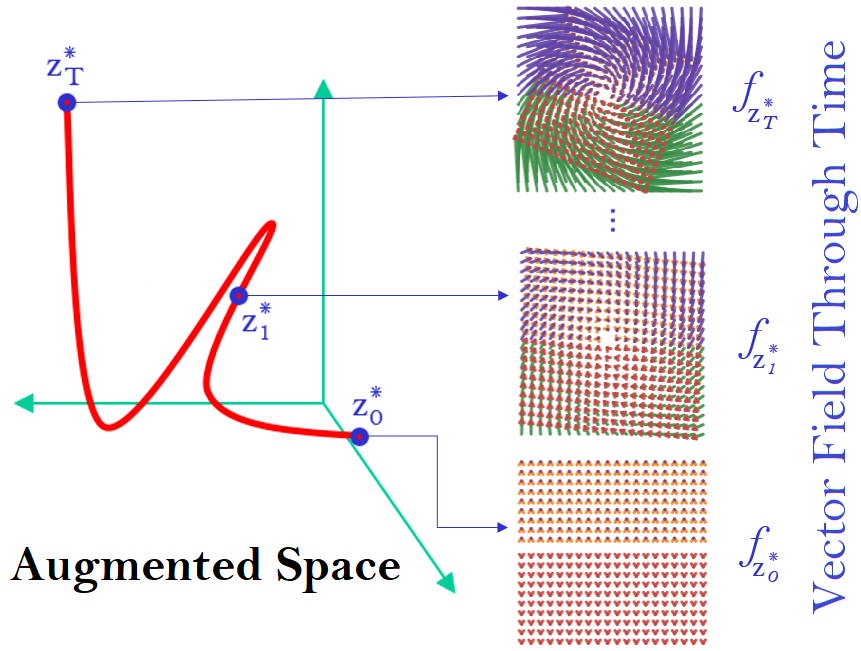}
\end{tabular}
\caption{Vector field evolution in AFFJORD as a function of the augmented dimensions $\boldsymbol{z}^{*}_t$.}
\label{vec_field_evol}
\end{figure}

In \cite{anode_dupont2019}, it is shown that there exist functions which neural ODEs are not capable of representing. To tackle this issue and enhance the expressiveness of the neural ODE transformation, they propose to lift the data into a higher dimensional space, on which the neural ODE is applied, to subsequently project the output back to the original space. Such augmented neural ODEs (ANODEs) have been experimentally shown to lead to improved flexibility and generalisation properties than the non-augmented counterpart.


Inspired by \cite{anode_dupont2019}, in this paper we develop a theoretical framework to increase the flexibility of ODE flows, such as FFJORD, through dimension augmentation. 
To this end, we derive an explicit formula for Jacobians corresponding to neural ODE transformations.
This formula represents the continuous generalization of the chain rule, which we name \emph{the cable rule}, that ultimately allows the derivation of the Jacobian expression and determinant for the composition of the operations defining the ANODE flow: augmentation, neural ODE transformation, and projection.  

To enable computational feasibility and ensure that the ANODE transformation is diffeomorphic, we allow the augmented dimensions to 
parameterize the vector field acting on the original dimensions (but not vice-versa). We name the resulting model augmented FFJORD (AFFJORD). In AFFJORD, the evolution of time is high dimensional (Figure \ref{vec_field_evol}), and is learnt via the vector field defined by the augmented component, contrasting the linear time evolution of FFJORD. This setup coincides with the framework introduced by \cite{zhang_2019}, but in the context of normalizing flows. Our experiments on 2D data of toy distributions and image datasets such as MNIST, CIFAR-10 and CelebA ($32\times32$), show that the proposed ANODE flow defined by AFFJORD outperforms FFJORD in terms of density estimation, thus highlighting the improved flexibility and representation properties of the underlying vector field. 

\section{Background and Related Work}

\textbf{Normalizing Flows:} The Normalizing Flow framework was previously defined in \cite{tabak_eijden,tabak_cristina}, and was popularised by \cite{rezende_NF}  and
by \cite{nice_2015} respectively in the context of variational inference and density estimation.
\newline
A Normalizing Flow is a transformation defined by a sequence of invertible and differentiable functions mapping a simple base probability distribution (e.g., a standard normal) into a more complex one. Let $\boldsymbol{Z}$ and $\boldsymbol{X}=g(\boldsymbol{Z})$ be random variables where $g$ is a diffeomorphism with inverse $h$. If we denote their probability density functions by $f_{\boldsymbol{Z}}$ and $f_{\boldsymbol{X}}$, based on the change of variable theorem we get:
\begin{equation}\label{s2}
f_{\boldsymbol{X}}(\boldsymbol{x})=f_{\boldsymbol{Z}}(\boldsymbol{z})|\frac{d\boldsymbol{z}}{d\boldsymbol{x}}|=f_{\boldsymbol{Z}}(h(\boldsymbol{x}))|\frac{d h(\boldsymbol{x})}{d\boldsymbol{x}}|.
\end{equation}
In general, we want to optimize the parameters of $h$ such that we maximize the likelihood of sampled points ${\boldsymbol{x}}_1,...,{\boldsymbol{x}}_n$. Once these parameters are optimized, then we can give as input any test point ${\boldsymbol{x}}$ on the right hand side of Equation \ref{s2}, and calculate its likelihood. For the generative task, being able to easily recover $g$ from $h$ is essential, as the generated point ${\boldsymbol{x}}_g$ will take form ${\boldsymbol{x}}_g$=$g(\boldsymbol{z}_s)$, where $\boldsymbol{z}_s$ is a sampled point from the base distribution $f_{\boldsymbol{Z}}$.
\newline
For increased modeling flexibility, we can use a chain (flow) of transformations, $\boldsymbol{z}_i=g_i(\boldsymbol{z}_{i-1}), i\in[n]$. In this case due to chain rule we have:
\begin{equation}\label{s3}
f_{\boldsymbol{Z}_n}(\boldsymbol{z}_n)=f_{\boldsymbol{Z}_0}(\boldsymbol{z}_0)|\frac{d\boldsymbol{z}_0}{d\boldsymbol{z}_n}|=f_{\boldsymbol{Z}_0}(h_0(...h_n(\boldsymbol{z}_n)))\prod_{i=1}^{n}| \frac{d h_i(\boldsymbol{z}_i)}{d\boldsymbol{z}_{i}}|
\end{equation}
The interested reader can find a more in-depth review of normalizing flows in \cite{Kobyzev_2021} and \cite{jmlr_summ_NF}.
\newline
\newline
\textbf{Neural ODE Flows:}  Neural ODEs, \cite{node_chen2018}, are continuous generalizations of residual networks:

\begin{equation}\label{neural_ODEs}
\boldsymbol{z}_{t_{i+1}}=\boldsymbol{z}_i+\epsilon f(\boldsymbol{z}_{t_i}, t_i, \boldsymbol{\theta}) \rightarrow \boldsymbol{z}(t)=\boldsymbol{z}(0) +\int_0^t f(\boldsymbol{z}(\tau),\tau,\boldsymbol{\theta}) d\tau \text{, as } \epsilon \rightarrow 0.
\end{equation}

In \cite{pontrjagin1962,node_chen2018}, it is shown that the gradients of neural ODEs can be computed via the adjoint method, with constant memory cost with regards to "depth":
\begin{equation}\label{adjoint_method}
       \frac{dL}{d \boldsymbol{\theta}}=
      \int_0^T \frac{\partial L}{\partial \boldsymbol{z}(t)} \frac{\partial f(\boldsymbol{z}(t), t, \boldsymbol{\theta})}{\partial \boldsymbol{\theta}}dt=-\int_T^0 \frac{\partial L}{\partial \boldsymbol{z}(t)} \frac{\partial f(\boldsymbol{z}(t), t, \boldsymbol{\theta})}{\partial \boldsymbol{\theta}}dt,
\end{equation}
where $\frac{\partial L}{\partial \boldsymbol{z}(t)}$ can be calculated simultaneously by
\begin{equation}\label{adjoint_method2}
    \frac{\partial L}{\partial \boldsymbol{z}(t)}=\frac{\partial L}{\partial \boldsymbol{z}(T)}-\int_T^t \frac{\partial L}{\partial \boldsymbol{z}(\tau)}\frac{\partial f(\boldsymbol{z}(\tau),\tau,\boldsymbol{\theta}(\tau)) }{\partial \boldsymbol{z}(\tau)}d\tau.
\end{equation}

In addition, they also derive the expression for the instantaneous change of variable, which enables one to train continuous normalizing flows:
\begin{equation}\label{instantaneous_change}
    \log{p(\boldsymbol{z}(0))}=\log{p(\boldsymbol{z}(T))}+\int_0^T tr{\frac{\partial  f(\boldsymbol{z}(t), t, \boldsymbol{\theta})}{\partial \boldsymbol{z}(t)}dt},
\end{equation}
where $\boldsymbol{z}(0)$ represents a sample from the data. 
A surprising benefit is that one does not need to calculate the determinant of the Jacobian of the transformation anymore, but simply the trace of a matrix. These models are collectively called Neural ODE flows (NODEFs) or simply Continuous Normalizing Flows (CNFs). In \cite{ffjord_chen2019}, these ideas are further explored and computational simplifications are introduced, notably the use of Hutchinson’s trace estimator, \cite{hutchinson_trick_est,adams_2018}, as an unbiased stochastic estimator of the trace in the likelihood expression in Equation \ref{instantaneous_change}. The resulting model is named FFJORD.
\newline
\newline
\textbf{Multiscale Architectures:}  In \cite{real_nvp}, a multiscale architecture for normalizing flows is implemented, which transforms the data shape from $[c,s,s]$ to $[4c,\frac{s}{2},\frac{s}{2}]$, where $c$ is the number of channels and $s$ is the height and width of the image. Effectively this operation trades spatial size for additional channels, and after each transformation a normalizing flow is applied on half the channels, while the other half are saved in an array. This process can be repeated as many times as the width and height of the transformed image remain even numbers. In the end, all saved channels are concatenated to construct an image with the original dimensions. A visual description of this process can be found in Figure \ref{multiscale}.
\begin{figure*}
\includegraphics[width=1\textwidth]{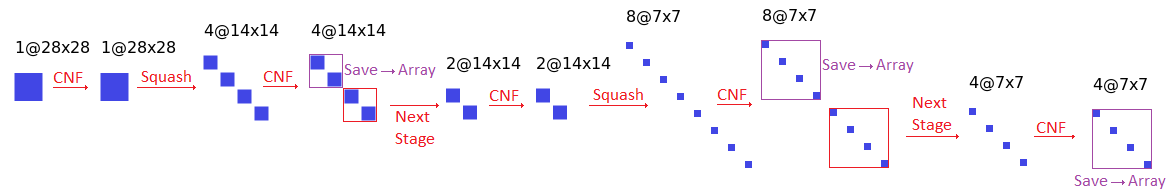}
\caption{Example of the multiscale architecture on MNIST dataset}
\label{multiscale}
\end{figure*}
\newline
\newline
\textbf{Augmented Neural ODEs:}  Considering that transformations by neural ODEs are diffeomorphisms (hence homeomorphisms), \cite{anode_dupont2019} show that Neural Ordinary Differential Equations (NODEs) learn representations that preserve the topology of the input space, and prove that this implies the
existence of functions that Neural ODEs cannot represent. To address these limitations, they introduce the Augmented Neural ODEs:
\begin{equation*}
    \frac{d}{dt}\begin{bmatrix}
\boldsymbol{z}(t)\\
\boldsymbol{z}^{*}(t)
\end{bmatrix}=h(\begin{bmatrix}
\boldsymbol{z}(t)\\
\boldsymbol{z}^{*}(t)
\end{bmatrix},t)=\begin{bmatrix}
f(\boldsymbol{z}(t),\boldsymbol{z}^*(t),\boldsymbol{\theta})\\
g(\boldsymbol{z}(t),\boldsymbol{z}^*(t),\boldsymbol{\theta})
\end{bmatrix}
\end{equation*}
\begin{equation}\label{augmented_Neural_ODEs}
\text{ for }\begin{bmatrix}
\boldsymbol{z}(0)\\
\boldsymbol{z}^{*}(0)
\end{bmatrix}=\begin{bmatrix}
\boldsymbol{x}\\
0
\end{bmatrix},
\end{equation}
where $\boldsymbol{z}^*(t)$ is the augmented component, and $h=[f,g]$ is the vector field to be learnt.

In addition to being more expressive models, \cite{anode_dupont2019} show that augmented neural ODEs are empirically more stable, generalize better and have a lower computational cost than Neural ODEs. A schematic of their architecture in the discrete case can be found in Figure \ref{dis_augNODE_archits}.

\begin{figure}
\centering
\begin{tabular}{ll}
\includegraphics[scale=0.3]{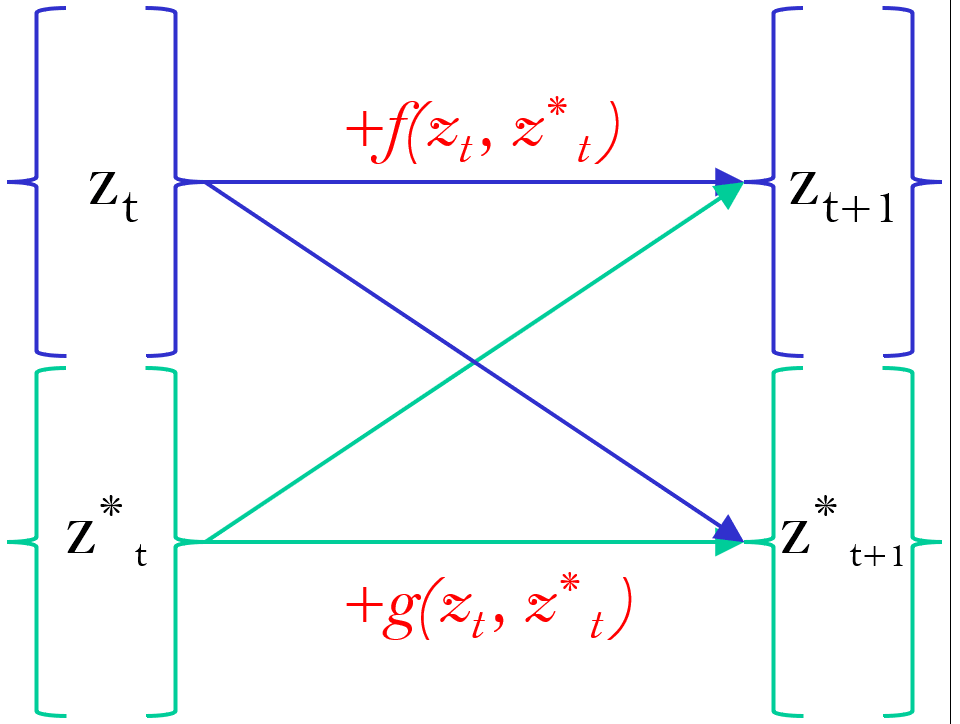}
\hfill
\includegraphics[scale=0.3]{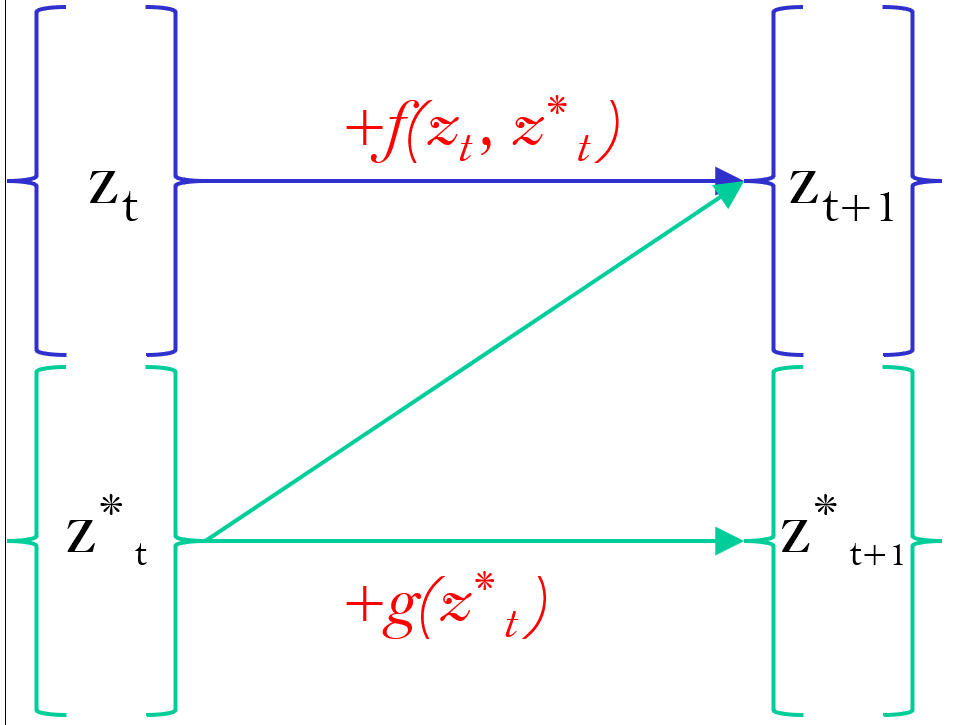}
\end{tabular}
\caption{Architecture of discretized augmented neural ODEs (left) and of the discretized augmented neural ODE flows implemented in AFFJORD (right).}
\label{dis_augNODE_archits}
\end{figure}

\section{Proposed Framework}

In the first subsection, we introduce our model AFFJORD, which is a special case of augmented neural ODEs. In the second subsection, we give the generalisation of the chain rule in the continuous sense which we refer to as the cable rule, which is analogous to forward sensitivity. Next, we give an intuitive proof of the continuous backpropagation, by showing its equivalence to the continuous generalisation of the total derivative decomposition. Then, using the cable rule, we give a more detailed explanation on why the instantaneous change of variable still holds in our model. In the final section, the augmented multiscale architecture is described, as it is implemented in the experimental section.

\subsection{Proposed Model: Augmented FFJORD (AFFJORD)}\label{aug_neural_ODE_flows}

A neural ODE of the type $f(\boldsymbol{z}(t),\boldsymbol{z}^*(t),\boldsymbol{\theta})$, where $\boldsymbol{z}^*(t)= \boldsymbol{z}^*(t)+ \int_0^t g(\boldsymbol{z}^*(\tau),\boldsymbol{\phi}) d\tau$,  can equivalently be written as $h_f(\boldsymbol{z}(t),t,\boldsymbol{\theta})=f(\boldsymbol{z}(t),\boldsymbol{z}^*(t),\boldsymbol{\theta})$. 
Thus, motivated by \cite{anode_dupont2019}, we propose to lift each data point $\boldsymbol{z}(0)$ $\in\mathbb{R}^n$ to a higher dimensional space $\mathbb{R}^{n+m}$, by augmentation via an $m$ dimensional vector $\boldsymbol{z}^*(0)$, which in practice is set to be the zero vector.
Therefore, the joint vector $[\boldsymbol{z}(0),\boldsymbol{z}^*(0)]$ is transformed by the vector field $h=(f,g)$:
\begin{equation}\label{augnodef_eq1r}
w(T)=\begin{bmatrix}
\boldsymbol{z}(T)\\
\boldsymbol{z}^{*}(T)
\end{bmatrix}=\begin{bmatrix}
\boldsymbol{z}(0)\\
\boldsymbol{z}^{*}(0)
\end{bmatrix}+\int_0^{T}\begin{bmatrix}
f(\boldsymbol{z}(t),\boldsymbol{z}^*(t),\boldsymbol{\theta})\\
g(\boldsymbol{z}^*(t),\boldsymbol{\phi})
\end{bmatrix}dt.
\end{equation}
We are not interested in $z^*(T)$ as our interest lies on the transformed $z(0)$, that is $z(T)$. Since by definition such coupled dynamics are contained in the formulation of neural ODEs, this implies that the instantaneous change of formula still holds, and the transformation is injective. This sort of augmentation can be seen as a special case of augmentation introduced in \cite{anode_dupont2019}, where the augmented dimensions depend only on themselves. The augmented dimensions $\boldsymbol{z}^*(t)$ can also be seen as time dependent weights of the non-autonomous $f$ whose evolution is determined by the autonomous ODE $g$, hence giving $f$ greater flexibility in time \cite{zhang_2019}.
\subsection{The Cable Rule}\label{subsec3.1}
If we define a chain of transformations
\begin{equation}\label{initialchain}
\boldsymbol{z}_i=g_i(\boldsymbol{z}_{i-1}) \text{ for } i \in \{1,...,n\},
\end{equation}
due to the chain rule we have:
\begin{equation*}
    \frac{d \boldsymbol{z}_n}{d \boldsymbol{z}_0}=\frac{\partial \boldsymbol{z}_n}{\partial \boldsymbol{z}_{n-1}}\frac{d \boldsymbol{z}_{n-1}}{d \boldsymbol{z}_{0}}=\frac{\partial \boldsymbol{z}_n}{\partial \boldsymbol{z}_{n-1}}\frac{\partial \boldsymbol{z}_{n-1}}{\partial \boldsymbol{z}_{n-2}}...\frac{\partial  \boldsymbol{z}_1}{\partial \boldsymbol{z}_0}=
\end{equation*}
\begin{equation}\label{cablerule1}
=\frac{\partial g_n(\boldsymbol{z}_{n-1})}{\partial \boldsymbol{z}_{n-1}}\frac{\partial g_{n-1}(\boldsymbol{z}_{n-2})}{\partial \boldsymbol{z}_{n-2}}...\frac{\partial g_1 (\boldsymbol{z}_0)}{\partial \boldsymbol{z}_0}.
\end{equation}
We can choose each $g_i$ to infinitesimally modify its input, i.e., $g_i(\boldsymbol{z}_{i-1})=\boldsymbol{z}_{i-1}+\epsilon f_i(\boldsymbol{z}_{i-1},t_{i-1},\boldsymbol{\theta})$, so that the chain in Expression \ref{initialchain} transforms $\boldsymbol{z}_0$ continuously. It is clear that $\boldsymbol{z}(t)=\boldsymbol{z}(0) +\int_0^t f(\boldsymbol{z}(t),t,\boldsymbol{\theta}) d\tau$ is the limit of the previous iterative definition when $\epsilon \rightarrow 0$. Then the expression $\frac{d \boldsymbol{z}_n}{d \boldsymbol{z}_0}$ in Equation \ref{cablerule1} converges to $\frac{d\boldsymbol{z}(t)}{d\boldsymbol{z}(0)}$, which as we show in Appendix \ref{appendixA}, satisfies the differential equation below:
\begin{equation}\label{cablerule1b}
\frac{d(\frac{d\boldsymbol{z}(t)}{d\boldsymbol{z}(0)})}{dt}=\frac{\partial f(\boldsymbol{z}(t))}{\partial \boldsymbol{z}(t)}\frac{d\boldsymbol{z}(t)}{d\boldsymbol{z}(0)}.
\end{equation}
Using the Magnus expansion \cite{magnus_1954,Blanes_2009}, we conclude that
\begin{equation}\label{cablerule2}
    \frac{d \boldsymbol{z}(T)}{d \boldsymbol{z}(0)}=e^{\boldsymbol{\Omega}(T)}, \text{ for\ \  } \boldsymbol{\Omega}(t)=\sum_{k=1}^{\infty}\boldsymbol{\Omega}_k(t),
\end{equation} 
where $\boldsymbol{\Omega}_i(t)$ are the terms of the Magnus expansion (See Appendix \ref{appendixA}). As this expression gives the generalisation of the chain rule in the continuous sense, we refer to it as the cable rule. We notice that Equation \ref{cablerule1b} gives the dynamics of the Jacobian of the state with respect to the initial condition $\boldsymbol{z(0)}$. The cable rule is therefore analogous to the forward sensitivity formula for ODEs which provides the dynamics of Jacobian of the state with respect to the parameters $\boldsymbol{\theta}$ of the flow  \cite{forward_sens2013}. This relation is highlighted and explained in more detail in Appendix A. On this note, in Appendix \ref{appendixG}, we derive the cable rule via Equation \ref{adjoint_method2}. Furthermore, in Appendix \ref{appendixE}, we derive the instantaneous change of variables from the cable rule. 


\subsection{The Continuous Generalisation of the Total Derivative Decomposition and Continuous Backpropagation}\label{adsens}
In the case that $f=f(\boldsymbol{x}(\boldsymbol{\theta}),\boldsymbol{y}(\boldsymbol{\theta}))$, then $\frac{df}{d\boldsymbol{\theta}}=\frac{\partial f}{\partial \boldsymbol{x}}\frac{d\boldsymbol{x}}{d\boldsymbol{\theta}}+\frac{\partial f}{\partial \boldsymbol{y}}\frac{d\boldsymbol{y}}{d\boldsymbol{\theta}}$, as $\boldsymbol{\theta}$ contributes to both $\boldsymbol{x}$ and $\boldsymbol{y}$, which in turn determine $f$.
If $\boldsymbol{z}(T)=\boldsymbol{z}(0)+\int_0^T f(\boldsymbol{z}(t),\boldsymbol{\theta}(t), t)dt$, then $\boldsymbol{\theta}$ controls the vector field at each time point during integration, hence we expect that these infinitesimal contributions of the transformation from $\boldsymbol{z}(0)$ to $\boldsymbol{z}(T)$ should be integrated. Indeed, as we prove in Appendix \ref{appendixB}, the following holds:
\begin{equation}\label{conttotderiv1}
      \frac{d\boldsymbol{z}(T)}{d \boldsymbol{\theta}}=\int_0^T \frac{\partial \boldsymbol{z}(T)}{\partial \boldsymbol{z}(t)} \frac{\partial f(\boldsymbol{z}(t),\boldsymbol{\theta} (t))}{\partial \boldsymbol{\theta}(t)}\frac{\partial \boldsymbol{\theta}(t)}{\partial \boldsymbol{\theta}}dt,
\end{equation} from which for  $\boldsymbol{\theta}(t)=\boldsymbol{\theta}$, and for some function $L=L(\boldsymbol{z}(T))$, we deduce:
\begin{equation}\label{conttotderiv2}
      \frac{dL}{d\boldsymbol{\theta}}=\frac{\partial L}{\partial \boldsymbol{z}(T)}\frac{d \boldsymbol{z}(T)}{d \boldsymbol{\theta}}=\int_0^T \frac{\partial L}{\partial \boldsymbol{z}(T)} \frac{\partial \boldsymbol{z}(T)}{\partial \boldsymbol{z}(t)} \frac{\partial f(\boldsymbol{z}(t),\boldsymbol{\theta})}{\partial \boldsymbol{\theta}}dt=\int_0^T \frac{\partial L}{\partial \boldsymbol{z}(t)}  \frac{\partial f(\boldsymbol{z}(t),\boldsymbol{\theta})}{\partial \boldsymbol{\theta}}dt
\end{equation}
thus giving an alternative and intuitive proof of continuous backpropagation \cite{pontrjagin1962,node_chen2018}.  In Appendix H, we show that the adjoint method can be used to prove Equation \ref{conttotderiv1}, hence the continuous total derivative decomposition is equivalent to continuous backpropagation. An alternate derivation of Equation \ref{conttotderiv1} can be found in the Appendix of \cite{dissecting_nodes}

\subsection{AFFJORD as a Special Case of Augmented Neural ODE Flows}\label{aug_neural_ODE_flows2}
As discussed in subsection \ref{aug_neural_ODE_flows}, the ODE dynamics in Equation \ref{augnodef_eq1r}, can be seen as a special case of the following joint ODE transformation:
\begin{equation}\label{augnodef_eq1}
w(T)=\begin{bmatrix}
\boldsymbol{z}(T)\\
\boldsymbol{z}^{*}(T)
\end{bmatrix}=\begin{bmatrix}
\boldsymbol{z}(0)\\
\boldsymbol{z}^{*}(0)
\end{bmatrix}+\int_0^{T}\begin{bmatrix}
f(\boldsymbol{z}(t),\boldsymbol{z}^*(t),\boldsymbol{\theta})\\
g(\boldsymbol{z}^*(t),\boldsymbol{z}(t),\boldsymbol{\phi})
\end{bmatrix}dt.
\end{equation}
In this general form, the model is unsuitable to be used in practice for two reasons:
\newline
1) The transformation of $\boldsymbol{z}(0)$ to $\boldsymbol{z}(T)$ is not necessarily injective. Indeed, the transformation from $[\boldsymbol{z}(0),\boldsymbol{z}^*(0)]$ to $[\boldsymbol{z}(T),\boldsymbol{z}^*(T)]$ is injective due to the Picard–Lindelöf Theorem, however, for two data points $\boldsymbol{z'}(0)$ and $\boldsymbol{z''}(0)$, their images $\boldsymbol{z'}(T)$, $\boldsymbol{z''}(T)$ might be identical as long as their respective augmented dimensions $\boldsymbol{z'}^{*}(T)$, $\boldsymbol{z''}^{*}(T)$ differ. 
\newline
2) The Jacobian determinant of this general transformation is computationally intractable. Using the chain rule we can express the Jacobian determinant of this transformation as
\begin{equation}\label{pre1}
\big|\frac{d \boldsymbol{z}(T)}{d \boldsymbol{z}(0)}\big|=
\big|\frac{d \boldsymbol{z}(T)}{d \begin{bmatrix}
\boldsymbol{z}(T),\boldsymbol{z}^{*}(T)
\end{bmatrix}}\frac{d \begin{bmatrix}
\boldsymbol{z}(T),
\boldsymbol{z}^{*}(T)
\end{bmatrix}}{d \begin{bmatrix}
\boldsymbol{z}(0),
\boldsymbol{z}^{*}(0)
\end{bmatrix}}\frac{d \begin{bmatrix}
\boldsymbol{z}(0),
\boldsymbol{z}^{*}(0)
\end{bmatrix}}{d \boldsymbol{z}(0)}\big|.
\end{equation}
The middle term on the RHS of Equation \ref{pre1} can be further developed via the cable rule, to give the expression of the determinant of the Jacobian of augmented neural ODE flows, which is not computationally feasible in general.
\newline
As explained in Section \ref{aug_neural_ODE_flows}, the special case (AFFJORD) formulated in Equation \ref{augnodef_eq1r}, mitigates the issues mentioned above. 
\newline
Regarding issue 1), we have the following:
\begin{proposition}\label{prop1} The architecture of AFFJORD ensures that the transformation is injective. 
\end{proposition}
\begin{proof}
Indeed, as $\boldsymbol{z}^*(t)$ is not dependent on external factors, and since $\boldsymbol{z}^*(0)$ is constant regarding $\boldsymbol{z}(0)$, the end result $\boldsymbol{z}^*(T)$ will always be the same. Hence, for $\boldsymbol{z'}(0)$ and $\boldsymbol{z''}(0)$ their images $\boldsymbol{z'}(T)$ and $\boldsymbol{z''}(T)$ must be different, since their equality would imply $[\boldsymbol{z'}(T),\boldsymbol{z}^{*}(T)]=[\boldsymbol{z''}(T),\boldsymbol{z}^{*}(T)]$ contradicting the Picard–Lindelöf Theorem.
\end{proof}
Issue 2) is mitigated as Equation \ref{pre1} simplifies to 
\begin{equation}\label{pre3b}
\big|\frac{d \boldsymbol{z}(T)}{d \boldsymbol{z}(0)}\big|=\big|\begin{bmatrix}
I,0
\end{bmatrix}
{ \begin{bmatrix}
e^{\int_0^T \frac{\partial f(\boldsymbol{z}(t),\boldsymbol{z}^*(t),\boldsymbol{\theta})}{\partial \boldsymbol{z}(t)} dt+\boldsymbol{\Omega}^{[z]}_2(T)+...}& \bar{\boldsymbol{B}}(T)\\
 0 & \bar{\boldsymbol{D}}(T)
\end{bmatrix}}
\begin{bmatrix}
I\\
0
\end{bmatrix}\big|,
\end{equation}
for two block matrices $\bar{\boldsymbol{B}}(T)$ and $\bar{\boldsymbol{D}}(T)$. This is proven in Appendix \ref{appendixF}. In this case, the $[z]$ in $\boldsymbol{\Omega}^{[z]}_2(T)$ denotes the restriction of the Magnus expansion to the original dimensions. In case that the base distribution is multivariate normal, from
\begin{equation*}
-\log{p(\boldsymbol{z}(0))}=-\log\big[{p(\boldsymbol{z}(T))}\ \big| e^{\int_0^T \frac{\partial f(\boldsymbol{z}(t),\boldsymbol{z}^*(t),\boldsymbol{\theta})}{\partial \boldsymbol{z}(t)} dt+...} \big|\big]
\end{equation*}
we derive the following loss function:
\begin{equation}\label{3.10}
L=\frac{||\boldsymbol{\boldsymbol{Z}}(T)||^2}{2}- {\int_{0}^{T}tr\frac{\partial f(\boldsymbol{z}(t),\boldsymbol{z}^*(t))}{\partial \boldsymbol{z}(t)}dt}.
\end{equation}





\subsection{Multiscale Architecture in Augmented Neural ODE Flows}
\begin{figure}
\includegraphics[width=1\textwidth]{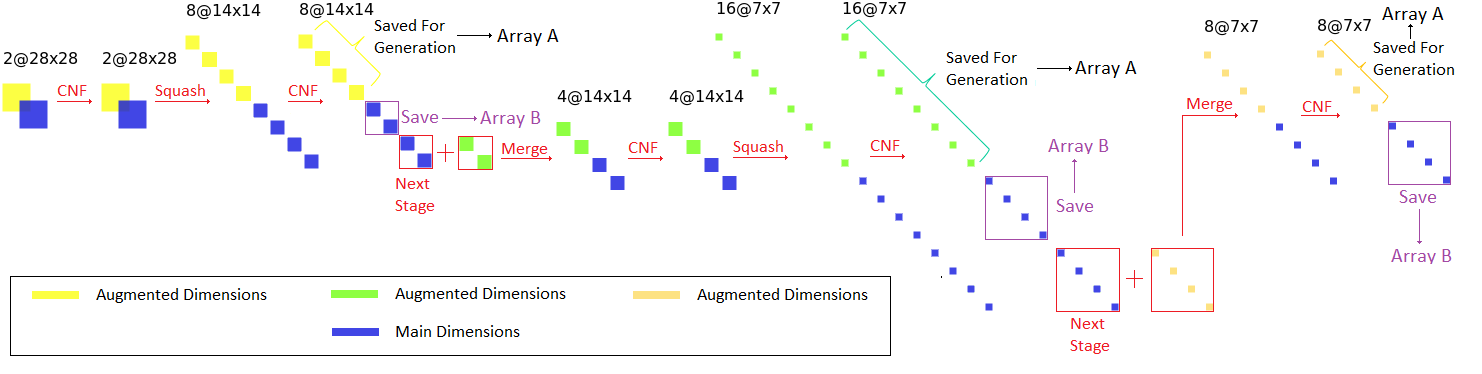}
\caption{Example of the augmented multiscale architecture on MNIST dataset}
\label{aug_multiscale}
\end{figure}

With reference to Figure \ref{aug_multiscale}, the multiscale architecture in the augmented case performs an Augmented Continuous Flow on the data as described in Subsection \ref{aug_neural_ODE_flows}, then squeezes the channels (the augmented as well as the original channels) as in the original multiscale architecture. However after the second transformation the augmented channels are removed and stored in a separated array (Array A). The data channels are treated as before, that is, half of them are saved (Array B), and the other half are squeezed again. After this process is finished we add new augmented channels, to repeat the cycle. Note that in order to generate data by the inverse transformation, we need to retrieve the transformed augmented dimensions previously stored (i.e., Array A).

\section{Experiments}

We compare the performance of AFFJORD with respect to the base FFJORD on 2D data of toy distributions, as well as on standard benchmark datasets such as MNIST, CIFAR-10 and CelebA($32\times32$). In the case of the 2D data, we use the implementation of CNFs provided in \cite{node_chen2018}, since using the Hutchinson's trace estimator and GPUs as in FFJORD provides no computational benefits in low dimensions. In this case we also use the non-adaptive Runge-Kutta 4 ODE solver, while for image data we use Dopri5, as well as the FFJORD implementation of \cite{ffjord_chen2019}. We use a batch size of $200$ for image data and a batch size of $512$ for the toy datasets. In the case of image data, we use a learning rate of $6\times 10^{-4}$, while for toy data the learning rate is set to $10^{-3}$. All experiments were performed on a single GPU. 
\subsection{Toy 2D Datasets}

\begin{figure}
\center
\includegraphics[width=0.85\textwidth]{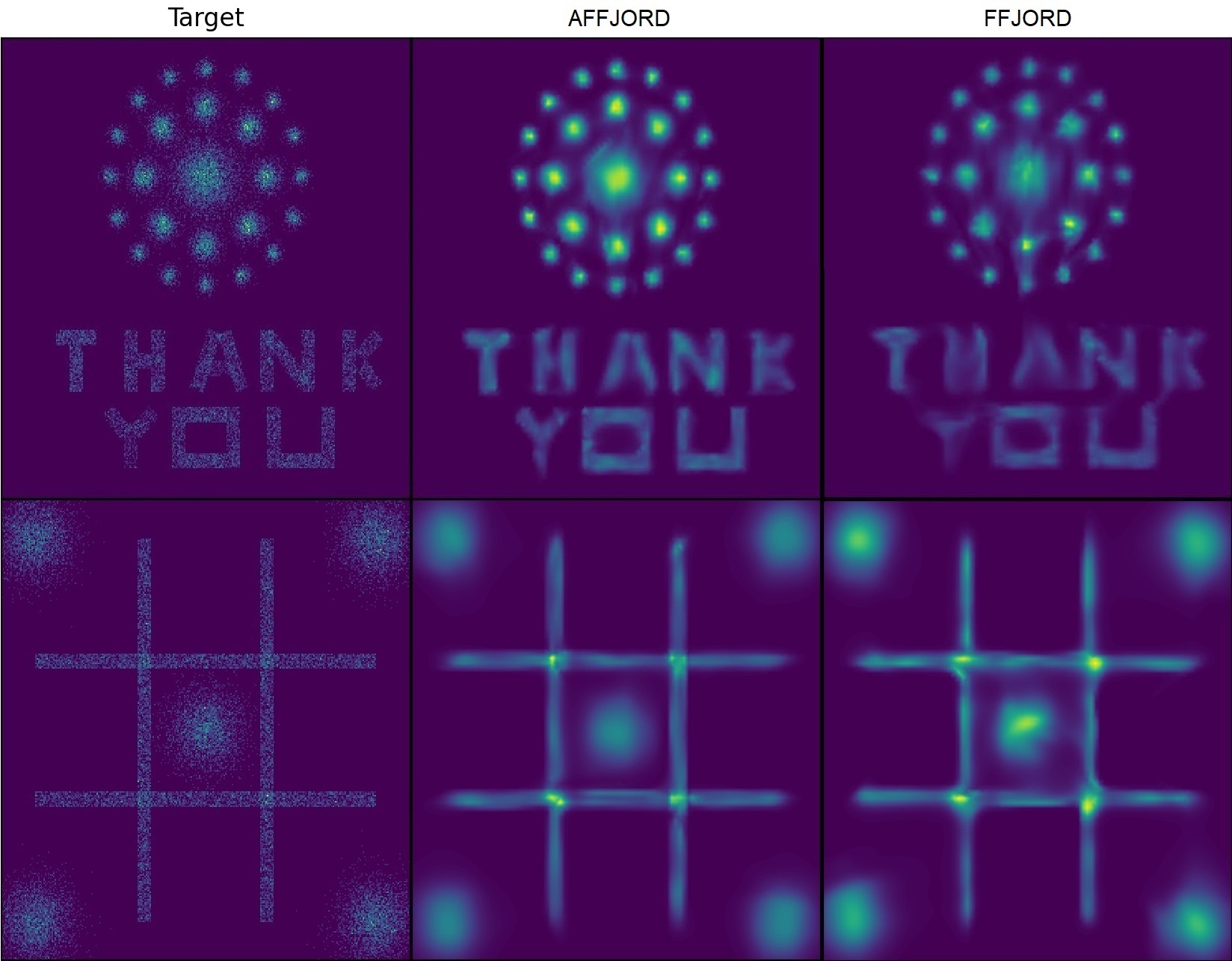}
\caption{Probability density modeling capabilities of  AFFJORD (second column) and FFJORD (third column) on 2D data of toy distributions.}
\label{toy2dimage}
\end{figure}

In order to visualise the performance of the model, we first test FFJORD and AFFJORD on 2D data of toy distributions, depicted in Figure \ref{toy2dimage}. In both cases we use the Runge-Kutta 4 solver with 40 time steps (160 function evaluations). In these examples, we used a hypernet architecture, where the augmented dimensions were fed to a hypernet (denoted as $hyp$) in order to generate the weights of the field of the main dimensions. The expression of the vector field to be learnt is the following: $\frac{d\boldsymbol{z}}{dt}=f\big([\boldsymbol{z}(t),\boldsymbol{z}^*(t)],\boldsymbol{\theta}(t)=hyp(\boldsymbol{z}^*(t),\boldsymbol{w})\big)$, where $\frac{d\boldsymbol{z}^*(t)}{dt}=g(\boldsymbol{z}^*(t),\boldsymbol{\phi})$. Thus, the learnable parameters are $\boldsymbol{w}$ and $\boldsymbol{\phi}$.

While both FFJORD and AFFJORD are capable of modelling multi-modal and discontinuous distributions, Figure \ref{toy2dimage} shows that AFFJORD has higher flexibility in modeling the complex data distributions considered, in comparison to FFJORD. In the first row, the target is the TY distribution, where the datapoints form letters and a cluster of Gaussian distributions. AFFJORD is more capable of separating the Gaussian spheres and modelling the shape of the text.

On the second row AFFJORD is capable of separating the Gaussian distribution in the center from the square that surrounds it. Furthermore, it separates the Gaussians in the corners from the hash symbol properly. In Figure \ref{toy2dgraphs}, we show the results of the validation loss per iteration for both models. We can notice that the loss of our model is roughly two standard deviations lower than the one of FFJORD. For each model the experiment was repeated 30 times. The experiments provided here show that AFFJORD is characterized by high flexibility of the vector field. Indeed, in FFJORD the vector field changes more slowly, whereas in AFFJORD the field is able to change almost abruptly, due to this greater flexibility in time \footnote{Videos showing the comparison of dynamics between FFJORD and AFFJORD can be found at https://imgur.com/gallery/kMGCKve}.
\begin{figure}
\centering
\begin{tabular}{ll}
\includegraphics[scale=0.42]{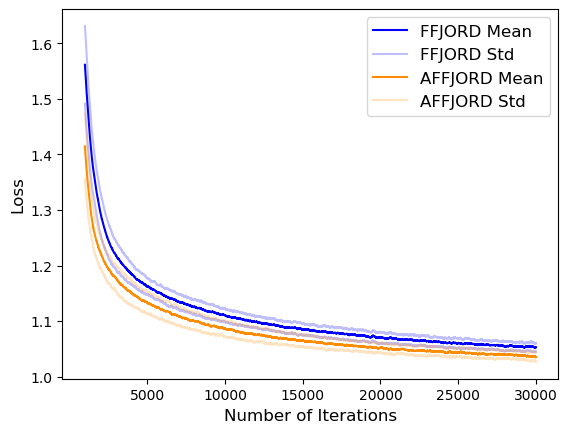}
&
\includegraphics[scale=0.42]{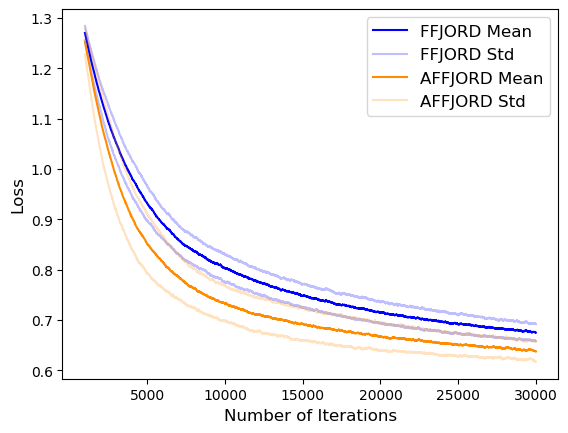}
\end{tabular}
\caption{Performance of FFJORD and AFFJORD on the Hash-Gaussian toy 2D dataset (left) and TY toy 2D dataset (right).}
\label{toy2dgraphs}
\end{figure}
It is important to emphasize that AFFJORD retains the ability to generate samples from the learnt distribution, by simply integrating in the opposite direction. Indeed, for a given sample $\boldsymbol{z}_s(T)$ from the base distribution, we augment it to $\boldsymbol{z}^*(T)$, as $\boldsymbol{z}^*(T)$ is the same for all data points. Then, we can simply integrate backwards this concatenated vector $[\boldsymbol{z}_s(T),\boldsymbol{z}^*(T)]$ to $[\boldsymbol{z}_s(0),\boldsymbol{z}^*(0)]$, drop the generated augmented dimensions $\boldsymbol{z}^*(0)$, and simply keep the generated data point $\boldsymbol{z}_s(0)$.

\subsection{Image Datasets}

We show that AFFJORD outperforms FFJORD on MNIST, CIFAR-10 and CelebA ($32\times32$). There are several architectures of FFJORD that can be used for this application. Out of the architectures that we tested, the one that performed best was the multiscale one, with three convolutional layers with 64 channels each. The number of CNF blocks was 1, and time was implemented by simply concatenating it as a channel into the data. 

\begin{table}[b]
  \caption{  Experimental Results for Density Estimation Models, in Bits/Dim for MNIST, CIFAR-10 and CelebA ($32\times32$). Lower Is Better. The Multiscale Architecture Is Used in All Cases.}

  \label{sample-table}
  \centering
  \begin{tabular}{llll}
    \\
       MODEL  & MNIST    & CIFAR10 & CelebA($32\times32$)\\
    \midrule
    Real NVP  & 1.06  & 3.49 &  -  \\
    Glow     & 1.05 & 3.35 &  - \\
    FFJORD   & 0.96 $\pm .00$   & 3.37 $\pm .00$ \ & 3.28$\pm .00$\\
    AFFJORD   & $\mathbf{0.95}\pm .01$      & $\mathbf{3.32}\pm .01$  & $\mathbf{3.23} \pm .00$\\
    \bottomrule
  \end{tabular}
\end{table}

When time was implemented via a hypernet, we observed that the training time increased and performance decreased, especially in the case of CIFAR-10. For AFFJORD we use the exact same base architecture, however, as in the case of 2D toy data, we enable the evolution of the vector fields through time via a hypernet which takes the augmented dimensions as an input, and outputs the weights of the main component. 
The augmented dimensions are concatenated as a channel to the main channel of the data. 
\begin{figure}
\centering
\begin{tabular}{ll}
\includegraphics[scale=0.44]{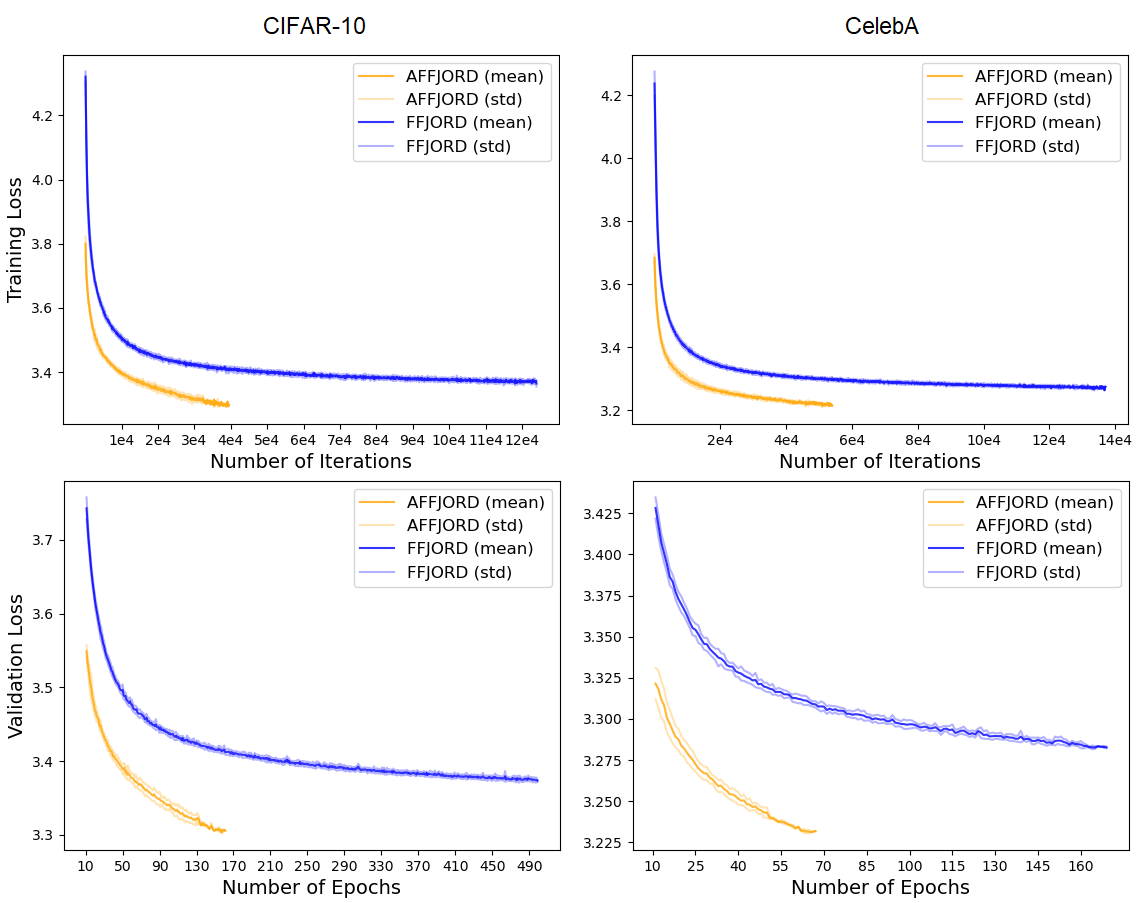}
\end{tabular}
\caption{Graphs of training and evaluation losses of FFJORD and AFFJORD on CIFAR-10 as well as CelebA ($32\times32$). We run the experiments 5 times. Lower is better.}
\label{all_realdata_losses}
\end{figure}
 The formulas for both the main field and the augmented component remain unchanged from the case of the 2D toy data, that is, $\frac{d\boldsymbol{z}}{dt}=f\big([\boldsymbol{z}(t),\boldsymbol{z}^*(t)],\boldsymbol{\theta}(t)=hyp(\boldsymbol{z}^*(t),\boldsymbol{w})\big)$ and  $\frac{d\boldsymbol{z}^*(t)}{dt}=g(\boldsymbol{z}^*(t),\boldsymbol{\phi})$. 

The main difference here is that $g(\boldsymbol{z}^*(t),\boldsymbol{\phi})$ is a fully connected network with one hidden layer, which does not take as input all the dimensions of $\boldsymbol{z}^*(t)$ but merely $20$ of them. Number $20$ was chosen as during fine-tuning, the best performance was reached in this setting.
The width of the hidden layer is also $20$, as it is the output. We fix $10$ of these $20$ dimensions and feed them to a linear hypernet with weight matrix shape $[10,p]$ to output the $p$ weights for the main component. It should be emphasized that the architecture of FFJORD in the main dimensions remains unchanged in AFFJORD for fair comparison, and we only fine-tuned the augmented structure in addition. Additional details about the experimental settings can be found in Appendix \ref{appendixI}.

As we show in Table \ref{sample-table}, AFFJORD slightly outperforms FFJORD on MNIST on the best run, since both models reach optimal performance, as seen from the generated samples in Figure \ref{samplesrealdata}.
However, our model outperforms FFJORD on the CIFAR-10 and CelebA ($32\times32$) dataset, as illustrated in Table \ref{sample-table}, as well as in Figure \ref{all_realdata_losses}. Based on the conducted experiments the farther FFJORD is from optimal performance, the larger the improvements brought by AFFJORD are. The calculation of results is done as in \cite{ffjord_chen2019}, where for each run the best evaluation result over epochs is taken. After 5 runs, for each model, the scores are averaged and reported in Table \ref{sample-table}.

\begin{figure}
\centering
\begin{tabular}{ll}
\includegraphics[scale=0.485]{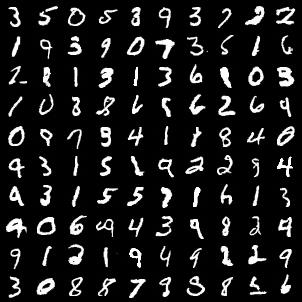}
&
\includegraphics[scale=0.43]{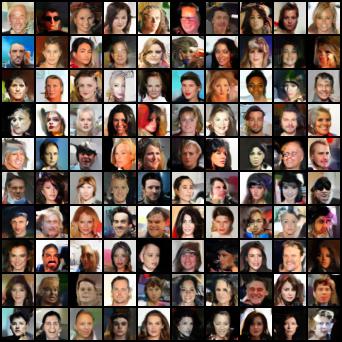}
\end{tabular}
\caption{Samples generated from AFFJORD: MNIST and CelebA $32\times32$.}
\label{samplesrealdata}
\end{figure}

In the case of MNIST, both models were trained for roughly 9 days, while in the case of CIFAR-10 and CelebA ($32\times32$), they were trained for approximately 14 days. The results corresponding to the Real NVP and Glow models, are taken from the original papers: \cite{real_nvp} and \cite{kingma_glow}.

As in the previous case, AFFJORD can generate samples by backintegrating. However, due to the use of the augmented multiscale architecture, where for each cycle we replace the augmented dimensions, these replaced augmented dimensions must be saved in an array for the backward generative pass. Examples of samples from AFFJORD are shown in Figure \ref{samplesrealdata} for both MNIST and CelebA ($32\times32$) datasets. Additional generated samples from both AFFJORD and FFJORD can be found in Appendix \ref{appendixD}.

\section{Limitations and Future Work}

\textbf{Number of Function Evaluations (NFE)}. As originally reported in \cite{ffjord_chen2019}, the number of function evaluations is one of the main bottlenecks of Neural ODE-based models. 
\begin{figure}
\centering
\begin{tabular}{l}
\includegraphics[scale=0.485]{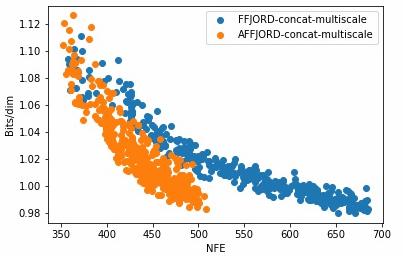}
\end{tabular}
\caption{Number of function evaluations of FFJORD and AFFJORD when the 'concatenate' architecture is used.}
\label{steps_loss_concat}
\end{figure}
Interestingly, if the concatenation architecture is used in AFFJORD, that is, we only replace the concatenated time channel in FFJORD with the augmented channel in AFFJORD, then the number of function evaluations decreases significantly. This aspect is illustrated in Figure \ref{steps_loss_concat}. However, the hypernet architecture of AFFJORD is prone to the issue of the number of function evaluations. Indeed, forward passes in AFFJORD-hypernet can be challenging, as the learnt vector field becomes too stiff with an increasing number of integration steps.
\newline
\newline
\textbf{Addition of Self-Attention } 
\cite{ho_flowplus} describe three modeling inefficiencies in prior work on flow models. One such factor is the lack of expressive power of convolutional layers used in normalizing flow models. Considering the performance improvements demonstrated in \cite{ho_flowplus}, in the future we intend to test the improvement in performance brought by the addition of self-attention in AFFJORD.
\newline
\newline
\textbf{Data Dependent Augmented Dimensions} 
As described in Section \ref{aug_neural_ODE_flows}, several simplifications are made in the architecture of the general augmented neural ODE flows, in order to ensure immediate bijectivity and reduce the computational complexity. However, other possible architectures exist, where for example $\boldsymbol{z}^*(0)$ is dependent on $\boldsymbol{z}(0)$, and $g(\boldsymbol{z}^*(t))=-\boldsymbol{z}^*(t)$. This would ensure that the augmented dimensions converge to zero, providing bijectivity. Since all augmented dimensions would be different during training, this would imply that the data is lifted to a higher plane, enabling richer transformations. Finding an approximation of the loss in Equation \ref{pre1} remains a challenge for the future however.
\newline
\newline
\textbf{Jacobian Regularization} 
Theoretically speaking, all performance enhancing modifications that can be applied to FFJORD are also applicable to AFFJORD. Such a modification which reduces the training time of FFJORD is presented in \cite{train_ode},  where both the vector field and its Jacobian are regularized. Thus, an interesting research direction in the future would be to test how the performance of AFFJORD is affected by such amendments.

\section{Conclusion} 
We have presented the generalization of the total derivative decomposition in the continuous sense as well as the continuous generalization of the chain rule, to which we refer as the cable rule. The cable rule is analogous to the forward sensitivity of ODEs in the sense that, it gives the dynamics of the Jacobian of the state with respect to the initial conditions, whereas forward sensitivity gives the dynamics of the Jacobian of the state with respect to the parameters of the flow. Motivated by this contribution, we propose a new type of continuous normalizing flow, namely Augmented FFJORD (AFFJORD), which outperforms the CNF state-of-art-approach, FFJORD, in the experiments we conducted on the task of density estimation on both 2D toy data, and on high dimensional datasets such as MNIST, CIFAR-10 and CelebA ($32\times32$). 


\newpage
\appendix

\section{Cable Rule: The Continuous Generalization of the Chain Rule}\label{appendixA}
We will first assume that $z$ is one dimensional. 
If $z_i=f_i(z_{i-1})$ for $i \in \{1,...,n\}$  then by the chain rule we have
\begin{equation}\label{a1}
    \frac{d z_n}{d z_0}=\frac{\partial z_n}{\partial z_{n-1}}\frac{d z_{n-1}}{d z_{0}}=\frac{\partial z_n}{\partial z_{n-1}}\frac{\partial z_{n-1}}{\partial z_{n-2}}...\frac{\partial  z_1}{\partial z_0}=\frac{\partial f_n(z_{n-1})}{\partial z_{n-1}}\frac{\partial f_{n-1}(z_{n-2})}{\partial z_{n-2}}...\frac{\partial f_1 (z_0)}{\partial z_0}
\end{equation}
Now, if we assume that $z_i=z(t_i)$  is transformed more gradually as in $z_{i+1}=z_i+\epsilon f(z_i)$, and that $t_{i+1}=t_i+\epsilon$, we get that 
\begin{equation}\label{a2}
    \frac{d z_n}{d z_0}=\frac{\partial z_n}{\partial z_{n-1}}\frac{\partial z_{n-1}}{\partial z_{n-2}}...\frac{\partial z_1}{\partial z_0}=(I+\epsilon \frac{\partial f(z_{n-1})}{\partial z_{n-1}})(I+\epsilon \frac{\partial f(z_{n-2})}{\partial z_{n-2}})...(I+\epsilon \frac{\partial f(z_0)}{\partial z_0})
\end{equation}
We see that $z(t)=z(0) +\int_0^t f(z(\tau)) d\tau$ is the limit of the previous iterative definition $z_{i+1}=z_i+\epsilon f(z_i, \boldsymbol{\theta}(t_i))$, when $\epsilon \rightarrow 0$. For simplicity, we have written $f(z_i)=f(z_i,\boldsymbol{\theta},t_i)$. If we decide to expand equation \ref{a2}, we obtain 
\begin{equation}\label{a3}
    \frac{d z_n}{d z_0}=I+\sum_{i=0}^{n-1}\frac{\partial f(z_i)}{\partial z_i}\epsilon+\sum_{i=0}^{n-1}\sum_{j<i}\frac{\partial f(z_i)}{\partial z_i}\frac{\partial f(z_j)}{\partial z_j}\epsilon^2+...+\frac{\partial f(z_0)}{\partial z_0}...\frac{\partial f(z_n)}{\partial z_n}\epsilon^n
\end{equation}
\begin{equation}\label{a4}
    =I+S_1^n+S_2^n+...+S_{n+1}^n,
\end{equation}
where 
\begin{equation}\label{a5}
    S_2^n=\sum_{i=0}^{n-1}\sum_{j<i}\frac{\partial f(z_i)}{\partial z_i}\frac{\partial f(z_j)}{\partial z_j}\epsilon^2,\
    S_3^n=\sum_{i=0}^{n-1}\sum_{j<i}\sum_{k<j}\frac{\partial f(z_i)}{\partial z_i}\frac{\partial f(z_j)}{\partial z_j}\frac{\partial f(z_k)}{\partial z_k}\epsilon^3,...
\end{equation}
\begin{figure}[ht] 
  
  \begin{minipage}[b]{0.6\linewidth}
    \includegraphics[width=.68\linewidth]{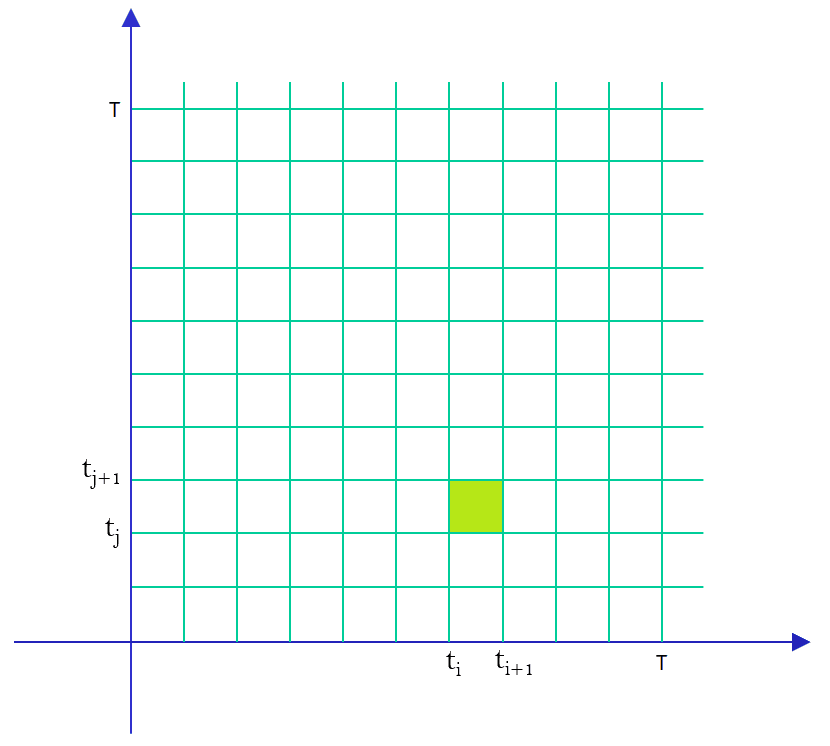} 
    \caption{Rectangle $(i,j)$ in discretized $[0,T]\times[0,T]$}
    \label{fig7} 
  \end{minipage}
  \begin{minipage}[b]{0.6\linewidth}
    \includegraphics[width=.68\linewidth]{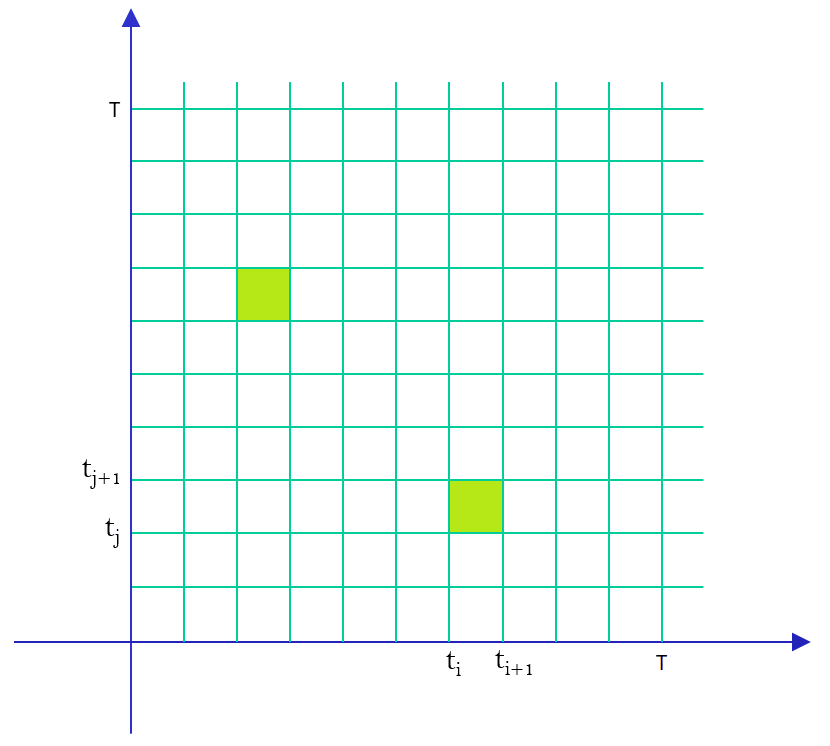} 
    \caption{Symmetry of $g_2^n$ \ \ \ \ \ \ \ \ \ \ \ \ \ \ \ \ }
    \label{fig8} 
  \end{minipage} 
  \begin{minipage}[b]{0.6\linewidth}
    \includegraphics[width=.68\linewidth]{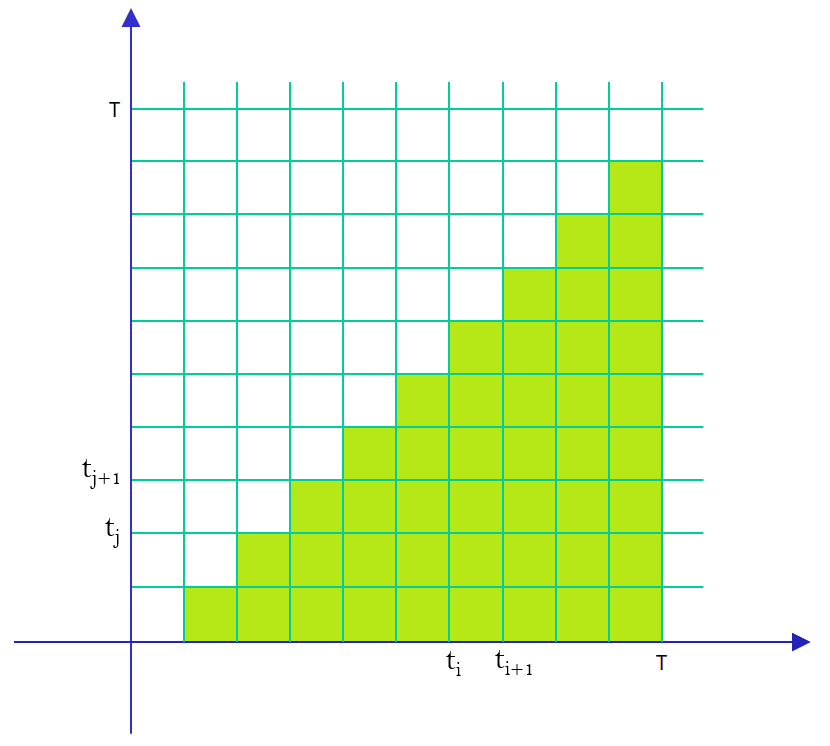} 
    \caption{Rectangles participating in $S_2^n$  \ \ \ \ \ \ \ \ \ \ \ \ \ } 
    \label{fig9}
  \end{minipage}
  \begin{minipage}[b]{0.6\linewidth}
    \includegraphics[width=.68\linewidth]{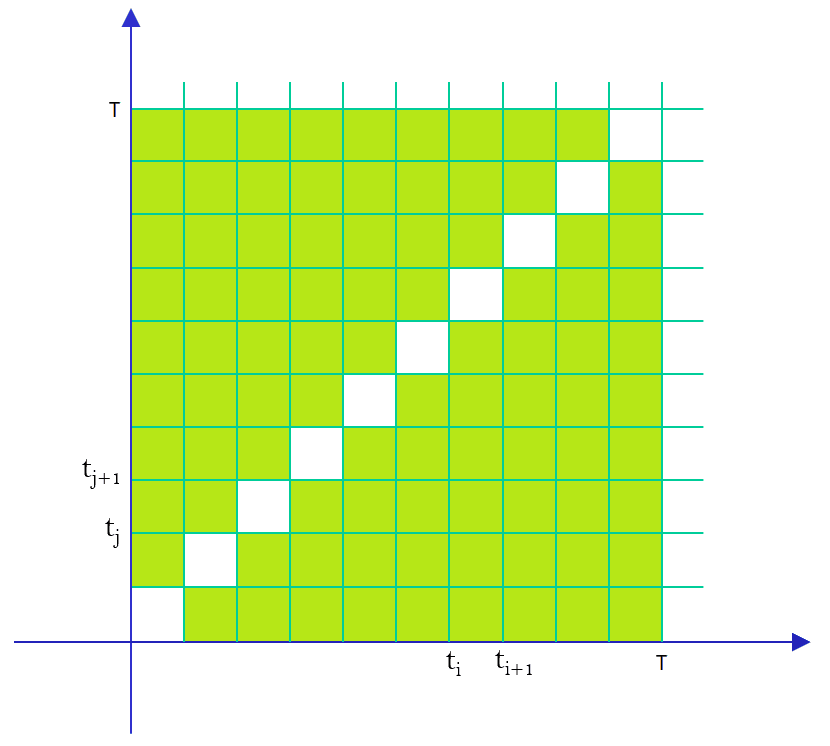} 
    \caption{Rectangles in $\bar{S}_2^n$ \ \ \ \ \ \ \ \ \ \ \ \ \ \ \ \ }
    \label{fig10} 

  \end{minipage} 
\end{figure}
We will focus on the sum $S_2^n$ for a moment. Let us define
\begin{equation}\label{a6}
    g_2^n(x,y)=\frac{\partial f(z_{t_i})}{\partial z_{t_i}}\frac{\partial f(z_{t_j})}{\partial z_{t_j}}, \text{ for } x\in [t_i,t_{i+1}], y\in [{t_j},{t_{j+1}}].
\end{equation}
It is clear that $g_2^n$ is the discretization of 
\begin{equation}\label{a7}
    g_2(t,u)=\frac{\partial f(z_{t})}{\partial z_{t}}\frac{\partial f(z_u)}{\partial z_u}, \text{ for } t\in [0,T], u\in [0,T].
\end{equation}
We can notice that $S_2^n=\sum_{i=0}^{n-1}\sum_{j<i}\ g_2^n(t_i,t_j)\epsilon^2$ and that $g_2^n(t_i,t_j)=g_2^n(t_j,t_i)$ is symmetric, but in $S_2^n$, only takes values in the rectangles under the diagonal as illustrated in Figure \ref{fig9}. If we decide to expand the sum $S_2^n$ on the rectangles above the diagonal as in Figure \ref{fig10} and denote it as $\bar{S}_2^n=\sum_{i=0}^{n-1}\sum_{j\neq i}\ g_2^n(t_i,t_j)\epsilon^2$, then due to the symmetry of $g_2^n(t_i,t_j)=g_2^n(t_j,t_i)$, we deduce that $S_2^n=\frac{1}{2} \bar{S}_2^n$. The only rectangles missing in $\bar{S}_2^n$, regarding the discretization of $[0,T]\times [0,T]$, are the ones corresponding to the cases when $i=j$, which are the rectangles in the diagonal. It is important to emphasize here that $S_2^n$ is the sum of $C_n^2$ terms, while $\bar{S}_2^n$ is made out of $V_n^2=2!C_n^2$ terms. This is especially apparent when we notice the discretization of $[0,T]\times [0,T]$ is composed of $n^2$ rectangles, while the diagonal is composed of $n$ rectangles, hence $\bar{S}_2^n$ is the sum of $n^2-n=V_n^2$ terms. The  ratio of the collective mass of the rectangles in the diagonal with respect to the entire $[0,T]\times [0,T]$ goes to zero as $n\rightarrow\infty$, hence we conclude that 
\begin{equation}\label{a8}
    \bar{S}_2^n=\sum_{i=0}^{n-1}\sum_{j\neq i}\frac{\partial f(z_i)}{\partial z_i}\frac{\partial f(z_j)}{\partial z_j}\epsilon^2 \rightarrow \int_{[0,T]\times[0,T]} g_2(t,u) d(t, u)=(\int_0^T \frac{\partial f(z_t)}{\partial z_t} dt)^2,
\end{equation}
thus,
\begin{equation}\label{a9}
    S_2^n=\frac{1}{2!}\bar{S}_2^n\rightarrow \frac{1}{2!}(\int_0^T \frac{\partial f(z_t)}{\partial z_t} dt)^2.
\end{equation}
Similarly for \begin{equation*}\label{doubledollar} S_3^n=\sum_{i=0}^{n-1}\sum_{j<i}\sum_{k<j}\frac{\partial f(z_i)}{\partial z_i}\frac{\partial f(z_j)}{\partial z_j}\frac{\partial f(z_k)}{\partial z_k}\epsilon^3,\end{equation*} we can define 
\begin{equation}\label{a10}
    g_3^n(x,y,z)=\frac{\partial f(z_{t_i})}{\partial z_{t_i}}\frac{\partial f(z_{t_j})}{\partial z_{t_j}}\frac{\partial f(z_{t_k})}{\partial z_{t_k}}, \text{ for } x\in [t_i,t_{i+1}], y\in [{t_j},{t_{j+1}}], z\in [{t_k},{t_{k+1}}],
\end{equation}
and 
\begin{equation}\label{a10a}
    g_3(t,u,v)=\frac{\partial f(z_{t})}{\partial z_{t}}\frac{\partial f(z_u)}{\partial z_u}\frac{\partial f(z_v)}{\partial z_v}, \text{ for } t\in [0,T], u\in [0,T], v\in [0,T].
\end{equation} In this case:
\begin{equation*}
    g_3^n(x,y,z)=g_3^n(x,y,z)=g_3^n(x,z,y)=g_3^n(y,x,z)=
\end{equation*}
\begin{equation}\label{a11}
   =g_3^n(y,z,x)=g_3^n(z,x,y)=g_3^n(z,y,x),
\end{equation} 
where each equality corresponds to one of the $6=3!$ permutations of $(x,y,z)$. As before we can expand the domain of $S_3^n$, by adding the rectangles in the discretization of $[0,T]\times [0,T]\times [0,T]$ such that $i$ might be smaller than $j$, as well as that $j$ could be smaller than $k$. We denote this expanded sum as $\bar{S}_3^n$, and by the symmetry of $g_3^n$, we notice that $S_3^n=\frac{1}{3!}\bar{S}_3^n$. This implies that the rectangles participating in $\bar{S}_3^n$, now cover most of $[0,T]\times [0,T]\times [0,T]$, where the only exceptions are the ones when $i=j$ or $j=k$ (or both). The number of rectangles participating in $\bar{S}_3^n$ is $V_n^3=3!C_n^3=n^3-3n^2+2n$, and since the number of all rectangles in $[0,T]\times [0,T]\times [0,T]$ is $n^3$, this implies that $3n^2-2n$ rectangles are missing in sum $\bar{S}_3^n$ from the cases when $i=j$ or $j=k$ (or both). The  ratio of the collective mass of such rectangles with respect to the entire $[0,T]\times [0,T]\times [0,T]$ goes to zero as $n\rightarrow\infty$, hence as before:
\begin{equation}\label{a12}
    \bar{S}_3^n=\sum_{i=0}^{n-1}\sum_{j\neq i}\sum_{i\neq k\neq j}\frac{\partial f(z_i)}{\partial z_i}\frac{\partial f(z_j)}{\partial z_j}\frac{\partial f(z_k)}{\partial z_k}\epsilon^3 \rightarrow \int_{[0,T]\times[0,T]\times[0,T]} g_3(t,u,v) d(t, u, v)
\end{equation}
\begin{equation}\label{a13}
            =(\int_0^T \frac{\partial f(z_t)}{\partial z_t} dt)^3,
\end{equation}
implying 
\begin{equation}\label{a14}
    S_3^n=\frac{1}{3!}\bar{S}_3^n \rightarrow \frac{1}{3!} (\int_0^T \frac{\partial f(z_t)}{\partial z_t} dt)^3.
\end{equation}
In a similar fashion we can prove that $S_k^n\rightarrow \frac{1}{k!} (\int_0^T \frac{\partial f(z_t)}{\partial z_t} dt)^k$.
Thus we conclude that:
\begin{equation}\label{a15}
    \frac{d z(T)}{d z(0)}=\frac{1}{0!}I+\frac{1}{1!}\int_{0}^{T}\frac{\partial f(z_t)}{\partial z_t}dt+\frac{1}{2!}(\int_{0}^{T}\frac{\partial f(z_t)}{\partial z_t}dt)^2+\frac{1}{3!}(\int_{0}^{T}\frac{\partial f(z_t)}{\partial z_t}dt)^3+...
\end{equation}
\begin{equation}\label{a16}
    \frac{d z(T)}{d z(0)}=e^{\int_{0}^{T}\frac{\partial f(z_t)}{\partial z_t}dt}.
\end{equation}
Unfortunately, this result does not hold when the dimensionality of $\boldsymbol{z}(t)$ is larger than one. Indeed, in this case, $\frac{\partial f(\boldsymbol{z}_{t_i})}{\partial \boldsymbol{z}_{t_i}}$ is a matrix, hence $g_2^n(x,y)=\frac{\partial f(\boldsymbol{z}_{t_i})}{\partial \boldsymbol{z}_{t_i}}\frac{\partial f(\boldsymbol{z}_{t_j})}{\partial \boldsymbol{z}_{t_j}}$ is not necessarily symmetric, as the commutator $[\frac{\partial f(\boldsymbol{z}_{t_i})}{\partial \boldsymbol{z}_{t_i}},\frac{\partial f(\boldsymbol{z}_{t_j})}{\partial \boldsymbol{z}_{t_j}}]$ is not necessarily zero. For this reason, inspired by the previous result we try a different approach. Indeed, we can see from Equation \ref{a16} that $\frac{d z(T)}{d z(0)}$ is the solution of the following ODE:

\begin{equation}\label{a'1}
\frac{d(\frac{d z(t)}{d z(0)})}{dt}=\frac{\partial f(z(t))}{\partial z(t)}\frac{d z(t)}{d z(0)}
\end{equation}

as the initial condition $\frac{d z(t=0)}{d z(0)}=I$. Hence we wish to prove that $\frac{d\boldsymbol{z}(t)}{d\boldsymbol{z}(0)}$ satisfies the same ODE in higher dimensions as well.

We notice that we can write:
\begin{equation}\label{a17}
\boldsymbol{z}(t)=\boldsymbol{z}(0) +\int_0^t f(\boldsymbol{z}(\tau)) d\tau=g(\boldsymbol{z}(0),t),
\end{equation}
therefore,
\begin{equation*}
\frac{d(\frac{d\boldsymbol{z}(t)}{d\boldsymbol{z}(0)})}{dt}=\frac{\partial^2 g(\boldsymbol{z}(0),t)}{\partial t \partial \boldsymbol{z}(0)}=\frac{\partial^2 g(\boldsymbol{z}(0),t)}{\partial \boldsymbol{z}(0) \partial t}=
\end{equation*}
\begin{equation}\label{a18}
=\frac{d(\frac{dg(\boldsymbol{z}(0),t)}{dt})}{d\boldsymbol{z}(0)}=\frac{df(\boldsymbol{z}(t))}{d\boldsymbol{z}(0)}=\frac{df(\boldsymbol{z}(t))}{d\boldsymbol{z}(t)}\frac{d\boldsymbol{z}(t)}{d\boldsymbol{z}(0)}.
\end{equation}
We pause for a moment, in order to  highlight the similarity of expression \ref{a18} and the forward sensitivity:

\begin{equation}\label{a'2}
\frac{d(\frac{d\boldsymbol{z}(t)}{d\boldsymbol{\theta}})}{dt}=\frac{df(\boldsymbol{z}(t),t,\boldsymbol{\theta})}{d\boldsymbol{\theta}}=\frac{\partial f(\boldsymbol{z}(t),t,\boldsymbol{\theta})}{\partial \boldsymbol{z}(t)}\frac{d  \boldsymbol{z}(t)}{d \boldsymbol{\theta}}+\frac{\partial f(\boldsymbol{z}(t),t,\boldsymbol{\theta})}{\partial \boldsymbol{\theta}}.
\end{equation}

Now from Expression \ref{a18}, we infer that $\frac{d\boldsymbol{z}(t)}{d\boldsymbol{z}(0)}$ is the solution of the following linear ODE:
\begin{equation}\label{a19}
\frac{d(\frac{d\boldsymbol{z}(t)}{d\boldsymbol{z}(0)})}{dt}=\frac{\partial f(\boldsymbol{z}(t))}{\partial \boldsymbol{z}(t)}\frac{d\boldsymbol{z}(t)}{d\boldsymbol{z}(0)}.
\end{equation}
If we write $\boldsymbol{Y}(t)=\frac{d\boldsymbol{z}(t)}{d\boldsymbol{z}(0)}$, then the equation above becomes:
\begin{equation}\label{a20}
\frac{d\boldsymbol{Y}(t)}{dt}=\frac{\partial f(\boldsymbol{z}(t))}{\partial \boldsymbol{z}(t)}\boldsymbol{Y}(t).
\end{equation}
The general solution of first-order homogeneous linear ODEs is given in \cite{magnus_1954,Blanes_2009}, and in our case can be written as follows:
\begin{equation}\label{a21}
    \frac{d \boldsymbol{\boldsymbol{z}}(t)}{d \boldsymbol{\boldsymbol{z}}(0)}=e^{\boldsymbol{\Omega}(t)}\frac{d\boldsymbol{z}(0)}{d\boldsymbol{z}(0)}=e^{\boldsymbol{\Omega}(t)}, \text{ for\ \  } \boldsymbol{\Omega}(t)=\sum_{k=1}^{\infty}\boldsymbol{\Omega}_k(t),
\end{equation} 
where
\begin{align}\label{a22}
  \boldsymbol{\Omega}_1(t) &= \int_0^t \boldsymbol{A}(t_1)\,dt_1, \\
  \boldsymbol{\Omega}_2(t) &= \frac{1}{2} \int_0^t dt_1 \int_0^{t_1} dt_2 \, [\boldsymbol{A}(t_1), \boldsymbol{A}(t_2)], \\
  \boldsymbol{\Omega}_3(t) &= \ 
 \end{align} 
 \begin{equation}
     = \frac{1}{6} \int_0^t dt_1 \int_0^{t_1} dt_2 \int_0^{t_2} dt_3 \,
                 \Bigl(\big[\boldsymbol{A}(t_1), [\boldsymbol{A}(t_2), \boldsymbol{A}(t_3)]\big] + \big[\boldsymbol{A}(t_3), [\boldsymbol{A}(t_2), \boldsymbol{A}(t_1)]\big]\Bigr),
 \end{equation} and so on for $k>3$, where:
\begin{align}\label{a23}
 \boldsymbol{A}(t)=\frac{\partial f(\boldsymbol{z}(t),t,\boldsymbol{\theta})}{\partial \boldsymbol{\boldsymbol{z}}(t)}.
 \end{align}
We notice that if $z(t)$ is one dimensional, then this result agrees with the one in the previous approach. 
To conclude, we have proven the following theorem:
\begin{theorem}\label{otheorem1}
Let $\frac{d\boldsymbol{z}(t)}{dt}= f(\boldsymbol{z}(t),\boldsymbol{\theta},t)$, where $f$ is continuous in $t$ and Lipschitz continuous in $\boldsymbol{z}(t)$. Then the following holds: 
\begin{equation}\label{a27}
    \frac{d \boldsymbol{z}(T)}{d \boldsymbol{z}(0)}=e^{\int_{0}^{T}\frac{\partial f(\boldsymbol{z}_t)}{\partial \boldsymbol{z}_t}dt+\boldsymbol{\Omega}_2(T)+....},
\end{equation} where $\boldsymbol{\Omega}_{k>1}(t)$ are the terms of the Magnus series.
\end{theorem}
\begin{proof}
Since $f$ is continuous in $t$ and Lipschitz continuous in $\boldsymbol{z}$ then due to Picard–Lindelöf theorem, $\boldsymbol{z}(T)$ exists and is unique. Furthermore, since $f$ is  Lipschitz continuous in $\boldsymbol{z}(t)$, then $\frac{\partial f(\boldsymbol{z}_t)}{\partial \boldsymbol{z}_t}$ exists almost everywhere. Based on the previous derivations we reach the desired conclusion.
\end{proof}

\section{The Continuous Generalization of the Total Derivative Decomposition}\label{appendixB}
In Appendix \ref{appendixA}, we assumed that $\boldsymbol{z}(t)=\boldsymbol{z}(0) +\int_0^t f(\boldsymbol{z}(\tau),\boldsymbol{\theta},\tau) d\tau$, which was the the limit of $\epsilon \rightarrow 0$ of the previous iterative definition of $\boldsymbol{z}_{i+1}=\boldsymbol{z}_i+\epsilon f(\boldsymbol{z}_i,\boldsymbol{\theta},t_i)$. Now, we assume that the set of parameters in $f$ is different at each discrete time, that is $\boldsymbol{z}_{i+1}=\boldsymbol{z}_i+\epsilon f_i(\boldsymbol{z}_i, \boldsymbol{\theta}_i, t_i)$, for independent sets $\boldsymbol{\theta}_i$.
First, we notice that \begin{equation*}\label{doubledollar2}\boldsymbol{z}_n=\boldsymbol{z}_n(\boldsymbol{z}(t_0), \boldsymbol{\theta}(t_0),\boldsymbol{\theta}(t_1),...,\boldsymbol{\theta}(t_{n-1}))=\boldsymbol{z}_n(\boldsymbol{z}(t_1), \boldsymbol{\theta}(t_1),\boldsymbol{\theta}(t_2),...,\boldsymbol{\theta}(t_{n-1}))=\end{equation*} 
\begin{equation*}...\boldsymbol{z}_n(\boldsymbol{z}(t_{k+1}), \boldsymbol{\theta}(t_{k+1}),\boldsymbol{\theta}(t_{k+2}),...,\boldsymbol{\theta}(t_{n-1}))=...\end{equation*}
\begin{equation*}...=\boldsymbol{z}_n(\boldsymbol{z}(t_{n-1}), \boldsymbol{\theta}(t_{n-1}))=\boldsymbol{z}(t_n=T-\epsilon).\end{equation*}
 This can be seen from Figure \ref{conttotderivimg}.
\begin{figure}
\center
\includegraphics[width=0.5\textwidth]{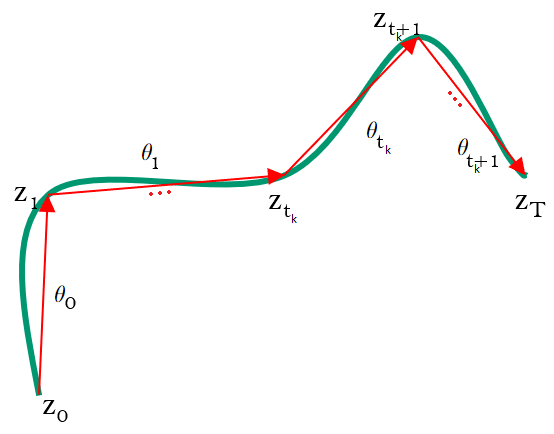}
\caption{The dependencies in the discretisation of $\frac{\boldsymbol{z}(t)}{t}=f(\boldsymbol{z}(t),\boldsymbol{\theta}(t),t)$. Variable $\boldsymbol{z}_{t_{k+1}}$ completely absorbs the contribution of $\boldsymbol{\theta}_{t_k}$ to $\boldsymbol{z}_T$, hence $\boldsymbol{z}_T$ is a function of $\boldsymbol{z}_{t_{k+1}}$ and the following sets of parameters $\boldsymbol{\theta}_{t_{k+1}},\boldsymbol{\theta}_{t_{k+2}},...,\boldsymbol{\theta}_{t_n}$.}
\label{conttotderivimg}
\end{figure}
Thus
\begin{equation}\label{d1}
      \frac{d \boldsymbol{z}_n}{d \boldsymbol{\theta}(t_k)}=\frac{d \boldsymbol{z}_n( \boldsymbol{z}(t_k+\epsilon), \boldsymbol{\theta}(t_k+\epsilon),\boldsymbol{\theta}(t_{k+2}),...,\boldsymbol{\theta}(t_n))}{d \boldsymbol{\theta}(t_k)}
\end{equation}

\begin{equation}\label{d2}
      \frac{d \boldsymbol{z}_n}{d \boldsymbol{\theta}(t_k)}=\frac{\partial \boldsymbol{z}_n}{\partial \boldsymbol{z}(t_k+\epsilon)}\frac{d \boldsymbol{z}(t_k+\epsilon)}{d \boldsymbol{\theta}(t_k)}+0+0+...+0=\frac{\partial \boldsymbol{z}_n}{\partial \boldsymbol{z}(t_k+\epsilon)}\frac{\partial \boldsymbol{z}(t_k+\epsilon)}{\partial \boldsymbol{\theta}(t_k)}.
\end{equation}
We now focus on analysing $\frac{\partial \boldsymbol{z}_n}{\partial \boldsymbol{z}(t_k+\epsilon)}$ and $\frac{\partial \boldsymbol{z}(t_k+\epsilon)}{\partial \boldsymbol{\theta}(t_k)}$.
First we notice that from:
\begin{equation}\label{d3a}
      \frac{\partial \boldsymbol{z}_n}{\partial \boldsymbol{z}(t_k)}=\frac{\partial \boldsymbol{z}_n}{\partial \boldsymbol{z}(t_k+\epsilon)}\frac{\partial \boldsymbol{z}(t_k+\epsilon)}{\partial \boldsymbol{z}(t_k)} = \frac{\partial \boldsymbol{z}_n}{\partial \boldsymbol{z}(t_k+\epsilon)}(I+\epsilon\frac{\partial f(\boldsymbol{z}(t_k),\boldsymbol{\theta}(t_k)) }{\partial \boldsymbol{z}(t_k)}+O(\epsilon^2))
\end{equation}
we get
\begin{equation}\label{d3b}
      \frac{\partial \boldsymbol{z}_n}{\partial \boldsymbol{z}(t_k+\epsilon)}= \frac{\partial \boldsymbol{z}_n}{\partial \boldsymbol{z}(t_k)}-\epsilon\frac{\partial \boldsymbol{z}_n}{\partial \boldsymbol{z}(t_k+\epsilon)}\frac{\partial f(\boldsymbol{z}(t_k),\boldsymbol{\theta}(t_k)) }{\partial \boldsymbol{z}(t_k)}+O(\epsilon^2).
\end{equation}
Regarding $\frac{\partial \boldsymbol{z}(t_k+\epsilon)}{\partial \boldsymbol{\theta}(t_k)}$ the following holds:
\begin{equation}\label{d4}
      \frac{\partial \boldsymbol{z}(t_k+\epsilon)}{\partial \boldsymbol{\theta}(t_k)}=\frac{\partial (\boldsymbol{z}(t_k)+f(\boldsymbol{z}(t_k),\boldsymbol{\theta} (t_k))\epsilon+O(\epsilon^2))}{\partial \boldsymbol{\theta}(t_k)}=0+\frac{\partial f(\boldsymbol{z}(t_k),\boldsymbol{\theta} (t_k))}{\partial \boldsymbol{\theta}(t_k)}\epsilon+O(\epsilon^2).
\end{equation}
After combining both Equation \ref{d3b} and Equation \ref{d4} in expression \ref{d2}, we deduce that:
\begin{equation*}
      \frac{d \boldsymbol{z}_n}{d \boldsymbol{\theta}(t_k)}=
\end{equation*}
\begin{equation}\label{d5}
      \big(\frac{\partial \boldsymbol{z}_n}{\partial \boldsymbol{z}(t_k)}-\epsilon\frac{\partial \boldsymbol{z}_n}{\partial \boldsymbol{z}(t_k+\epsilon)}\frac{\partial f(\boldsymbol{z}(t_k),\boldsymbol{\theta}(t_k)) }{\partial \boldsymbol{z}(t_k)}+O(\epsilon^2)\big) \big(\frac{\partial f(\boldsymbol{z}(t_k),\boldsymbol{\theta} (t_k))}{\partial \boldsymbol{\theta}(t_k)}\epsilon+O(\epsilon^2)\big),
\end{equation}
thus
\begin{equation}\label{d6}
      \frac{\partial \boldsymbol{z}_n}{\partial \boldsymbol{\theta}(t_k)}=\frac{d \boldsymbol{z}_n}{d \boldsymbol{\theta}(t_k)}=\frac{\partial \boldsymbol{z}_n}{\partial \boldsymbol{z}(t_k)} \frac{\partial f(\boldsymbol{z}(t_k),\boldsymbol{\theta} (t_k))}{\partial \boldsymbol{\theta}(t_k)}\epsilon+ O(\epsilon^2)
\end{equation}
Now assuming that $\boldsymbol{\theta}(t_k)=g(t_k,\boldsymbol{\theta})$, we can write the total derivative of $\boldsymbol{z}_n$ with respect to parameters $\boldsymbol{\theta}$ :
\begin{equation}\label{d7}
      \frac{d\boldsymbol{z}_n}{d \boldsymbol{\theta}}=\sum_{k=0}^n \frac{\partial \boldsymbol{z}_n}{\partial \boldsymbol{\theta}(t_k)}\frac{d \boldsymbol{\theta}(t_k)}{d \boldsymbol{\theta}}=\sum_{k=0}^n \frac{\partial \boldsymbol{z}_n}{\partial \boldsymbol{z}(t_k)} \frac{\partial f(\boldsymbol{z}(t_k),\boldsymbol{\theta} (t_k))}{\partial \boldsymbol{\theta}(t_k)}\frac{d \boldsymbol{\theta}(t_k)}{d \boldsymbol{\theta}}\epsilon+ O(\epsilon^2)
\end{equation}
hence taking $\epsilon \rightarrow 0$, we conclude that
\begin{equation}\label{d8}
      \frac{d\boldsymbol{z}(T)}{d \boldsymbol{\theta}}=\int_0^T \frac{\partial \boldsymbol{z}(T)}{\partial \boldsymbol{z}(t)} \frac{\partial f(\boldsymbol{z}(t),\boldsymbol{\theta} (t))}{\partial \boldsymbol{\theta}(t)}\frac{\partial \boldsymbol{\theta}(t)}{\partial \boldsymbol{\theta}}dt.
\end{equation}
We can see that we are integrating the infinitesimal contributions of parameters $\boldsymbol{\theta}$, at each time $t$. In case that $\boldsymbol{\theta}(t)=g(t,\boldsymbol{\theta})=\boldsymbol{\theta}$, then we have 
\begin{equation*}
      \frac{d \boldsymbol{z}(T)}{d \boldsymbol{\theta}}=\int_0^T \frac{\partial \boldsymbol{z}(T)}{\partial \boldsymbol{z}(t)} \frac{\partial f(\boldsymbol{z}(t),\boldsymbol{\theta})}{\partial \boldsymbol{\theta}}\frac{\partial \boldsymbol{\theta}}{\partial \boldsymbol{\theta}}dt=\int_0^T \frac{\partial \boldsymbol{z}(T)}{\partial \boldsymbol{z}(t)} \frac{\partial f(\boldsymbol{z}(t),\boldsymbol{\theta})}{\partial \boldsymbol{\theta}}Idt=
\end{equation*}
\begin{equation}\label{d9}
=\int_0^T \frac{\partial \boldsymbol{z}(T)}{\partial \boldsymbol{z}(t)} \frac{\partial f(\boldsymbol{z}(t),\boldsymbol{\theta})}{\partial \boldsymbol{\theta}}dt=-
      \int_T^0 \frac{\partial \boldsymbol{z}(T)}{\partial \boldsymbol{z}(t)} \frac{\partial f(\boldsymbol{z}(t),\boldsymbol{\theta})}{\partial \boldsymbol{\theta}}dt.
\end{equation}
\begin{theorem}\label{otheorem2}
Let $\frac{dz(t)}{dt}= f(\boldsymbol{z}(t),\boldsymbol{\theta}(t),t)$, where $f$ is continuous in $t$ and Lipschitz continuous in $\boldsymbol{z}(t)$ and $\boldsymbol{\theta}(t)$. Then the following holds: 
\begin{equation}\label{d10}
      \frac{d\boldsymbol{z}(T)}{d \boldsymbol{\theta}}=\int_0^T \frac{\partial \boldsymbol{z}(T)}{\partial \boldsymbol{z}(t)} \frac{\partial f(\boldsymbol{z}(t),\boldsymbol{\theta} (t), t)}{\partial \boldsymbol{\theta}(t)}\frac{\partial \boldsymbol{\theta}(t)}{\partial \boldsymbol{\theta}}dt.
\end{equation}
\end{theorem}
\begin{proof}
As before, since $f$ is continuous in $t$ and Lipschitz continuous in $z$ then due to Picard–Lindelöf theorem, $z(T)$ exists and is unique. From Theorem \ref{otheorem1}, we can establish the existence of $    \frac{d \boldsymbol{z}(T)}{d \boldsymbol{z}(t)}$. Furthermore, since $f$ is  Lipschitz continuous in $\theta(t)$, then $\frac{\partial f(\boldsymbol{\theta}(t))}{\partial \boldsymbol{\theta}(t)}$ exists almost everywhere. Based on the previous derivations we reach the desired conclusion.
\end{proof}

\section{Generalization of Continuous Backpropagation into Piecewise Continuous Backpropagation}\label{appendixC}
We define $\boldsymbol{z}(t)$ as before, with the only difference being that its derivative is discontinuous at a point (say $\frac{T}{2}$):
\begin{align}\label{e1}
   \boldsymbol{z}(t)=
    \begin{cases}
    \boldsymbol{z}(0)+\int_0^t f(\boldsymbol{z}(\tau),\boldsymbol{\theta}(\tau))d\tau,\ t\in[0,\frac{T}{2}]
      \\
    \boldsymbol{z}(0)+\int_0^\frac{T}{2} f(\boldsymbol{z}(\tau),\boldsymbol{\theta}(\tau))d\tau+\int_\frac{T}{2}^t g(\boldsymbol{z}(\tau),\boldsymbol{\phi}(\tau))d\tau, \ t\in(\frac{T}{2},T]
    \end{cases}  
\end{align}
As in \cite{node_chen2018}, we can get
\begin{align}\label{e2}
   \frac{d (\frac{\partial L}{\partial \boldsymbol{z}(t)})}{dt}=
    \begin{cases}
    -\frac{\partial L}{\partial \boldsymbol{z}(t)}\frac{\partial f(\boldsymbol{z}(t),\boldsymbol{\theta}(t)) }{\partial \boldsymbol{z}(t)},\ t\in[0,\frac{T}{2}]
      \\
    -\frac{\partial L}{\partial \boldsymbol{z}(t)}\frac{\partial g(\boldsymbol{z}(t),\boldsymbol{\phi}(t)) }{\partial \boldsymbol{z}(t)},\ t\in(\frac{T}{2},T],
    \end{cases}  
\end{align}
thus
\begin{align}\label{e3}
   \frac{\partial L}{\partial \boldsymbol{z}(t)}=
    \begin{cases}
    \frac{\partial L}{\partial \boldsymbol{z}(T)}-\int_T^t \frac{\partial L}{\partial \boldsymbol{z}(\tau)}\frac{\partial g(\boldsymbol{z}(\tau),\boldsymbol{\phi}(\tau)) }{\partial \boldsymbol{z}(\tau)}d\tau,\ t\in(\frac{T}{2},T] 
      \\
    \frac{\partial L}{\partial \boldsymbol{z}(T)}-\int_\frac{T}{2}^t \frac{\partial L}{\partial \boldsymbol{z}(\tau)}\frac{\partial f(\boldsymbol{z}(\tau),\boldsymbol{\theta}(\tau)) }{\partial \boldsymbol{z}(\tau)}d\tau - \int_T^\frac{T}{2} \frac{\partial L}{\partial \boldsymbol{z}(\tau)}\frac{\partial g(\boldsymbol{z}(\tau)\boldsymbol{\phi}(\tau)) }{\partial \boldsymbol{z}(\tau)}d\tau,\ t\in(0,\frac{T}{2}].
    \end{cases}  
\end{align}
The approach developed in Appendix \ref{appendixB} can be used to generalize continuous backpropagation into piecewise continuous backpropagation. Indeed, identically as before, we have the first order appoximation:
\begin{align}\label{e5}
   \frac{d L_\epsilon}{d \boldsymbol{\theta}}=\sum_{k=0}^n \frac{\partial L_\epsilon}{\partial \boldsymbol{z}(t_k)} \frac{\partial f(\boldsymbol{z}(t_k),\boldsymbol{\theta} (t_k))}{\partial \boldsymbol{\theta}(t_k)}\frac{\partial \boldsymbol{\theta}(t_k)}{\partial \boldsymbol{\theta}}\epsilon+\sum_{k=0}^n \frac{\partial L_\epsilon}{\partial \boldsymbol{z}(t_k)} \frac{\partial g(\boldsymbol{z}(t_k),\boldsymbol{\phi} (t_k))}{\partial \boldsymbol{\phi}(t_k)}\frac{\partial \boldsymbol{\phi}(t_k)}{\partial \boldsymbol{\theta}}\epsilon
\end{align}
Taking the limit we get:
\begin{equation}\label{e6}
      \frac{dL}{d \boldsymbol{\theta}}=\int_0^\frac{T}{2}\frac{\partial L}{\partial \boldsymbol{z}(t)} \frac{\partial f(\boldsymbol{z}(t),\boldsymbol{\theta} (t))}{\partial \boldsymbol{\theta}(t)}\frac{\partial \boldsymbol{\theta}(t)}{\partial \boldsymbol{\theta}}dt+\int_\frac{T}{2}^T \frac{\partial L}{\partial \boldsymbol{z}(t)} \frac{\partial g(\boldsymbol{z}(t),\boldsymbol{\phi} (t))}{\partial \boldsymbol{\theta}(t)}\frac{\partial \boldsymbol{\phi}(t)}{\partial \boldsymbol{\theta}}dt
\end{equation}
and in the same manner we get:
\begin{equation}\label{e7}
      \frac{dL}{d \boldsymbol{\phi}}=\int_0^\frac{T}{2}\frac{\partial L}{\partial \boldsymbol{z}(t)} \frac{\partial f(\boldsymbol{z}(t),\boldsymbol{\theta} (t))}{\partial \boldsymbol{\theta}(t)}\frac{\partial \boldsymbol{\theta}(t)}{\partial \boldsymbol{\phi}}dt+\int_\frac{T}{2}^T\frac{\partial L}{\partial \boldsymbol{z}(t)} \frac{\partial g(\boldsymbol{z}(t),\boldsymbol{\phi} (t))}{\partial \boldsymbol{\theta}(t)}\frac{\partial \boldsymbol{\phi}(t)}{\partial \boldsymbol{\phi}}dt,
\end{equation}
If we set $\boldsymbol{\theta}(t)=\boldsymbol{\theta}$ and $\boldsymbol{\phi}(t)=\boldsymbol{\phi}$, then we get: 

\begin{equation}\label{e8}
      \frac{dL}{d \boldsymbol{\theta}}=\int_0^\frac{T}{2}\frac{\partial L}{\partial \boldsymbol{z}(t)} \frac{\partial f(\boldsymbol{z}(t),\boldsymbol{\theta} )}{\partial \boldsymbol{\theta}}dt,
\end{equation}
and
\begin{equation}\label{e9}
      \frac{dL}{d \boldsymbol{\phi}}=\int_\frac{T}{2}^T\frac{\partial L}{\partial \boldsymbol{z}(t)} \frac{\partial g(\boldsymbol{z}(t),\boldsymbol{\phi})}{\partial \boldsymbol{\phi}}dt.
\end{equation}

\section{Additional Generated Samples}\label{appendixD}
Below in Figure \ref{high_mnist_sample}, Figure \ref{high_cifar_sample} and Figure \ref{high_celeba_sample} , are presented additional examples of generated samples of MNIST, CIFAR-10 and CelebA($32\times32$) by AFFJORD, respectively. 
\begin{figure}
\center
\includegraphics[width=0.55\textwidth]{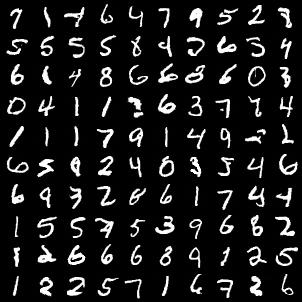}
\caption{Generated samples of MNIST by AFFJORD}
\label{high_mnist_sample}
\end{figure}

\begin{figure}
\center
\includegraphics[width=0.55\textwidth]{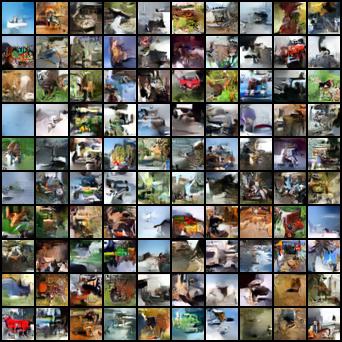}
\caption{Generated samples of CIFAR-10 by AFFJORD}
\label{high_cifar_sample}
\end{figure}

\begin{figure}
\center
\includegraphics[width=0.55\textwidth]{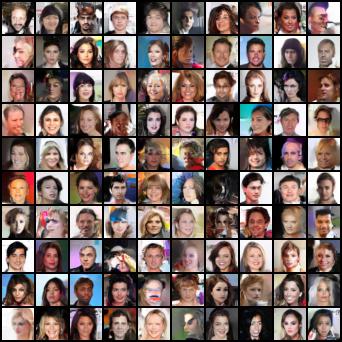}
\caption{Generated samples of CelebA($32\times32$) by AFFJORD}
\label{high_celeba_sample}
\end{figure}

In addition in Figure \ref{high_mnist_sample_FFJORD}, Figure \ref{high_cifar_sample_FFJORD} and Figure \ref{high_celeba_sample_FFJORD}, are presented additional examples of generated samples of MNIST, CIFAR-10 and CelebA($32\times32$) by FFJORD, respectively.
\begin{figure}
\center
\includegraphics[width=0.55\textwidth]{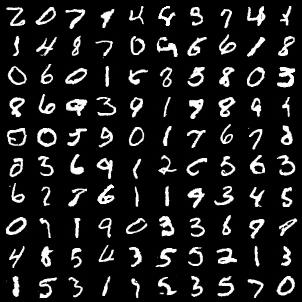}
\caption{Generated samples of MNIST by FFJORD}
\label{high_mnist_sample_FFJORD}
\end{figure}

\begin{figure}
\center
\includegraphics[width=0.55\textwidth]{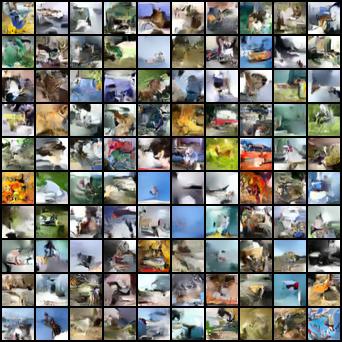}
\caption{Generated samples of CIFAR-10 by FFJORD}
\label{high_cifar_sample_FFJORD}
\end{figure}

\begin{figure}
\center
\includegraphics[width=0.55\textwidth]{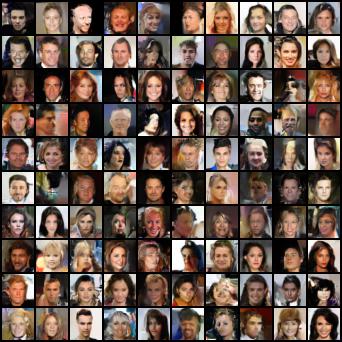}
\caption{Generated samples of CelebA ($32\times 32$) by FFJORD}
\label{high_celeba_sample_FFJORD}
\end{figure}

\section{Deriving the Instantaneous Change of Variable via the Cable Rule}\label{appendixE}
The trace of a commutator $[\boldsymbol{X},\boldsymbol{Y}]=\boldsymbol{X}\boldsymbol{Y}-\boldsymbol{Y}\boldsymbol{X}$ is always zero, hence we prove below that for all terms $\boldsymbol{\Omega}_{k>1}(t)$ given in Equations \ref{a22} in Appendix \ref{appendixA}, we have $tr(\boldsymbol{\Omega}_k(t))=0$. 

Indeed, since the trace and the integral are interchangeable we have:
\begin{equation*}
tr(\boldsymbol{\Omega}_2(t))=tr\frac{1}{2} \int_0^t dt_1 \int_0^{t_1} dt_2 \, [\boldsymbol{A}(t_1), \boldsymbol{A}(t_2)]
 \end{equation*}
 \begin{equation}\label{ee1}
=\frac{1}{2} \int_0^t dt_1 \int_0^{t_1} dt_2 \, tr[\boldsymbol{A}(t_1), \boldsymbol{A}(t_2)]=0.
 \end{equation}

 \begin{equation*}
   tr(\boldsymbol{\Omega}_3(t)) = \frac{1}{6} \int_0^t dt_1 \int_0^{t_1} dt_2 \int_0^{t_2} dt_3 \,
                 \Bigl(tr \big[\boldsymbol{A}(t_1), [\boldsymbol{A}(t_2), \boldsymbol{A}(t_3)]\big] +
 \end{equation*}
  \begin{equation}\label{ee2}
  + tr\big[\boldsymbol{A}(t_3), [\boldsymbol{A}(t_2), \boldsymbol{A}(t_1)]\big]\Bigr)
 \end{equation}
  \begin{equation}\label{ee3}
 = \frac{1}{6} \int_0^t dt_1 \int_0^{t_1}(0+0) dt_2 =0,
 \end{equation}
  and so on for $k>3$.

 The only exception is the case when $k=1$:
 \begin{align}\label{cablerule6a}
tr(\boldsymbol{\Omega}_1(t))=tr\int_0^t \boldsymbol{A}(t_1) dt_1=\int_0^ttr \frac{\partial f(\boldsymbol{z}(t_1),t_1,\boldsymbol{\theta})}{\partial \boldsymbol{z}(t_1)}dt_1.
 \end{align}
Finally, using Jacobi's formula, we conclude that:
\begin{equation}\label{cablerule7a}
\log{|\frac{d \boldsymbol{z}(T)}{d \boldsymbol{z}(0)}|}=\log{|e^{\boldsymbol{\Omega}(T)}|}=\log{e^{tr(\boldsymbol{\Omega}(T))}}={\int_0^T tr \frac{\partial f(\boldsymbol{z}(t),t,\boldsymbol{\theta})}{\partial \boldsymbol{z}(t)}}dt.
\end{equation}

\section{Simplifications in Section \ref{aug_neural_ODE_flows}}\label{appendixF}
First we notice that if $z^*(0)$ does not depend on $z(0)$, then 
\begin{equation}\label{f1}
\frac{d \boldsymbol{z}(T)}{d \boldsymbol{z}(0)} =\begin{bmatrix}
I,0
\end{bmatrix}
e^{ \begin{bmatrix}
\int_0^T \frac{\partial f(\boldsymbol{z}(t),\boldsymbol{z}^*(t),\boldsymbol{\theta})}{\partial \boldsymbol{z}(t)} dt&\int_0^T \frac{\partial f(\boldsymbol{z}(t),\boldsymbol{z}^*(t),\boldsymbol{\theta})}{\partial \boldsymbol{z}^*(t)} dt\\
\int_0^T \frac{\partial g(\boldsymbol{z}(t),\boldsymbol{z}^*(t),\boldsymbol{\phi})}{\partial \boldsymbol{z}(t)} dt&\int_0^T \frac{\partial g(\boldsymbol{z}(t),\boldsymbol{z}^*(t),\boldsymbol{\phi})}{\partial \boldsymbol{z}^*(t)} dt
\end{bmatrix}+\boldsymbol{\Omega}_2(T)+...}
\ \
\begin{bmatrix}
I\\
\frac{d\boldsymbol{z}^{*}(0)}{d \boldsymbol{z}(0)}
\end{bmatrix},
\end{equation}
becomes 
\begin{equation}\label{f2}
|\frac{d \boldsymbol{z}(T)}{d \boldsymbol{z}(0)}| =|\begin{bmatrix}
I,0
\end{bmatrix}
e^{ \begin{bmatrix}
\int_0^T \frac{\partial f(\boldsymbol{z}(t),\boldsymbol{z}^*(t),\boldsymbol{\theta})}{\partial \boldsymbol{z}(t)} dt& \int_0^T \frac{\partial f(\boldsymbol{z}(t),\boldsymbol{z}^*(t),\boldsymbol{\theta})}{\partial \boldsymbol{z}^*(t)} dt\\
\int_0^T \frac{\partial g(\boldsymbol{z}(t),\boldsymbol{z}^*(t),\boldsymbol{\phi})}{\partial \boldsymbol{z}(t)} dt & \int_0^T \frac{\partial g(\boldsymbol{z}(t),\boldsymbol{z}^*(t),\boldsymbol{\phi})}{\partial \boldsymbol{z}^*(t)} dt
\end{bmatrix}+\boldsymbol{\Omega}_2(T)+...}
\ \
\begin{bmatrix}
I\\
0
\end{bmatrix}|.
\end{equation}
On the other hand, if the augmented dimensions do not depend on the main dimensions then $\frac{\partial g(\boldsymbol{z}(t),\boldsymbol{z}^*(t),\boldsymbol{\phi})}{\partial \boldsymbol{z}(t)}=0$. This implies that
\begin{equation}\label{f3}
 \boldsymbol{A}(t)= \begin{bmatrix}
\frac{\partial f(\boldsymbol{z}(t),\boldsymbol{z}^*(t),\boldsymbol{\theta})}{\partial \boldsymbol{z}(t)} & \frac{\partial f(\boldsymbol{z}(t),\boldsymbol{z}^*(t),\boldsymbol{\theta})}{\partial \boldsymbol{z}^*(t)} \\
 \frac{\partial g(\boldsymbol{z}(t),\boldsymbol{z}^*(t),\boldsymbol{\phi})}{\partial \boldsymbol{z}(t)} & \frac{\partial g(\boldsymbol{z}(t),\boldsymbol{z}^*(t),\boldsymbol{\phi})}{\partial \boldsymbol{z}^*(t)} 
\end{bmatrix} 
\rightarrow 
\begin{bmatrix}
\frac{\partial f(\boldsymbol{z}(t),\boldsymbol{z}^*(t),\boldsymbol{\theta})}{\partial \boldsymbol{z}(t)} & \frac{\partial f(\boldsymbol{z}(t),\boldsymbol{z}^*(t),\boldsymbol{\theta})}{\partial \boldsymbol{z}^*(t)} \\
 0 & \frac{\partial g(\boldsymbol{z^*}(t),\boldsymbol{\phi})}{\partial \boldsymbol{z}^*(t)} 
\end{bmatrix}. 
\end{equation}
Hence,
\begin{equation}\label{f4}
\boldsymbol{\Omega}_1(t)=\int_0^t dt_1  \begin{bmatrix}
\frac{\partial f(\boldsymbol{z}(t_1),\boldsymbol{z}^*(t_1),\boldsymbol{\theta})}{\partial \boldsymbol{z}(t_1)} & \frac{\partial f(\boldsymbol{z}(t_1),\boldsymbol{z}^*(t_1),\boldsymbol{\theta})}{\partial \boldsymbol{z}^*(t_1)} \\
 0& \frac{\partial g(\boldsymbol{z^*}(t_1),\boldsymbol{\phi})}{\partial \boldsymbol{z}^*(t_1)} 
\end{bmatrix}=\begin{bmatrix}\int_0^t dt_1
\frac{\partial f(\boldsymbol{z}(t_1),\boldsymbol{z}^*(t_1),\boldsymbol{\theta})}{\partial \boldsymbol{z}(t_1)} & \boldsymbol{B}_1(t) \\
 0& \boldsymbol{D}_1(t)
\end{bmatrix}
\end{equation}
$\ \ \ \ \ \ \boldsymbol{\Omega}_2(t)=$
\begin{equation*}\label{f5}
\int_0^t dt_1 \int_0^{t_1} dt_2 \, \begin{bmatrix}
\frac{\partial f(\boldsymbol{z}(t_2),\boldsymbol{z}^*(t_2),\boldsymbol{\theta})}{\partial \boldsymbol{z}(t_2)} & \frac{\partial f(\boldsymbol{z}(t_2),\boldsymbol{z}^*(t_2),\boldsymbol{\theta})}{\partial \boldsymbol{z}^*(t_2)} \\
 0 & \frac{\partial g(\boldsymbol{z^*}(t_2),\boldsymbol{\phi})}{\partial \boldsymbol{z}^*(t_2)} 
\end{bmatrix}\begin{bmatrix}
\frac{\partial f(\boldsymbol{z}(t_1),\boldsymbol{z}^*(t_1),\boldsymbol{\theta})}{\partial \boldsymbol{z}(t_1)} & \frac{\partial f(\boldsymbol{z}(t_1),\boldsymbol{z}^*(t_1),\boldsymbol{\theta})}{\partial \boldsymbol{z}^*(t_1)} \\
 0& \frac{\partial g(\boldsymbol{z^*}(t_1),\boldsymbol{\phi})}{\partial \boldsymbol{z}^*(t_1)} 
\end{bmatrix}
\end{equation*}
\begin{equation*}\label{f6}
-\int_0^t dt_1 \int_0^{t_1} dt_2 \, \begin{bmatrix}
\frac{\partial f(\boldsymbol{z}(t_1),\boldsymbol{z}^*(t_1),\boldsymbol{\theta})}{\partial \boldsymbol{z}(t_1)} & \frac{\partial f(\boldsymbol{z}(t_1),\boldsymbol{z}^*(t_1),\boldsymbol{\theta})}{\partial \boldsymbol{z}^*(t_1)} \\
 0& \frac{\partial g(\boldsymbol{z^*}(t_1),\boldsymbol{\phi})}{\partial \boldsymbol{z}^*(t_1)} 
\end{bmatrix}\begin{bmatrix}
\frac{\partial f(\boldsymbol{z}(t_2),\boldsymbol{z}^*(t_2),\boldsymbol{\theta})}{\partial \boldsymbol{z}(t_2)} & \frac{\partial f(\boldsymbol{z}(t_2),\boldsymbol{z}^*(t_2),\boldsymbol{\theta})}{\partial \boldsymbol{z}^*(t_2)} \\
 0 & \frac{\partial g(\boldsymbol{z^*}(t_2),\boldsymbol{\phi})}{\partial \boldsymbol{z}^*(t_2)} 
\end{bmatrix}
\end{equation*}
\begin{equation*}\label{f7}
= \begin{bmatrix}
\int_0^t dt_1 \int_0^{t_1} dt_2 [\frac{\partial f(\boldsymbol{z}(t_1),\boldsymbol{z}^*(t_1),\boldsymbol{\theta})}{\partial \boldsymbol{z}(t_1)},\frac{\partial f(\boldsymbol{z}(t_2),\boldsymbol{z}^*(t_2),\boldsymbol{\theta})}{\partial \boldsymbol{z}(t_2)}] &\boldsymbol{B}_2(t) \\
 0& \boldsymbol{D}_2(t)
\end{bmatrix}=\begin{bmatrix}
 \boldsymbol{\Omega}_2^{[z]}(t) & \boldsymbol{B}_2(t) \\
 0 & \boldsymbol{D}_2(t)
\end{bmatrix}.
\end{equation*}
In this way we can prove that 
\begin{equation*}\label{f8}
\boldsymbol{\Omega}(t)=\sum_{k=1}^{\infty}\boldsymbol{\Omega}(t)_k=\begin{bmatrix}
\sum_{k=1}^{\infty} \boldsymbol{\Omega}^{[z]}_k(t) & \sum_{k=1}^{\infty}\boldsymbol{B}_k(t) \\
 0& \sum_{k=1}^{\infty}\boldsymbol{D}_k(t)\end{bmatrix}=\begin{bmatrix}
 \boldsymbol{\Omega}^{[z]}(t) & \boldsymbol{B}(t) \\
  0 & \boldsymbol{D}(t)
\end{bmatrix}.
\end{equation*}
Considering that the exponential of a matrix whose lower left block is zero, will still have a zero lower left block, we have:
\begin{equation*}\label{f9}
\frac{dz(t)}{dz(0)}=e^{\boldsymbol{\Omega}(t)}=\begin{bmatrix}
 e^{\boldsymbol{\Omega}^{[z]}(t)} & \bar{\boldsymbol{B}}(t) \\
 0 & \bar{\boldsymbol{D}}(t)
\end{bmatrix}.
\end{equation*}
Finally,
\begin{equation}\label{f10}
|\frac{d \boldsymbol{z}(T)}{d \boldsymbol{z}(0)}| =|\begin{bmatrix}
I,0
\end{bmatrix}
\begin{bmatrix}
 e^{\boldsymbol{\Omega}^{[z]}(t)} & \bar{\boldsymbol{B}}(t) \\
 0 & \bar{\boldsymbol{D}}(t)
\end{bmatrix}
\ \
\begin{bmatrix}
I\\
0
\end{bmatrix}|=|e^{\boldsymbol{\Omega}^{[z]}(t)}|=e^{\int_0^T tr\frac{\partial f(\boldsymbol{z}(t),\boldsymbol{z}^*(t),\boldsymbol{\theta})}{\partial \boldsymbol{z}(t)}dt+0+...0+...}.
\end{equation}


\section{Cable Rule Derived via Equation \ref{adjoint_method2}}\label{appendixG}

Differentiating Equation \ref{adjoint_method2} in Appendix \ref{appendixA}, we get:

\begin{equation}\label{g1}
\frac{d\big(\frac{dL}{d \boldsymbol{z}(t)}\big)}{dt}=-\frac{dL}{d \boldsymbol{z}(t)}\frac{\partial f(\boldsymbol{z}(t),t,\boldsymbol{\theta})}{\partial \boldsymbol{z}(t)}. 
\end{equation}

Choosing $L$ to be $\boldsymbol{z}(0)$, we have
\begin{equation}\label{g2}
\frac{d\big(\frac{d \boldsymbol{z}(0)}{d \boldsymbol{z}(t)}\big)}{dt}=-\frac{d \boldsymbol{z}(0)}{d \boldsymbol{z}(t)}\frac{\partial f(\boldsymbol{z}(t),t,\boldsymbol{\theta})}{\partial \boldsymbol{z}(t)}. 
\end{equation}

We define $\boldsymbol{A}(t):=\frac{d \boldsymbol{z}(0)}{d \boldsymbol{z}(t)}$, and $\boldsymbol{B}(t):=\boldsymbol{A}(t)^{-1}=\frac{d \boldsymbol{z}(t)}{d \boldsymbol{z}(0)}$, so that the equation above becomes 
\begin{equation}\label{g3}
\frac{d \boldsymbol{A}(t)}{dt}=-\boldsymbol{A}(t)\frac{\partial f(\boldsymbol{z}(t),t,\boldsymbol{\theta})}{\partial \boldsymbol{z}(t)},
\end{equation}
thus
\begin{equation}\label{g4}
-\boldsymbol{B}(t)\frac{d \boldsymbol{A}(t)}{dt}=\frac{\partial f(\boldsymbol{z}(t),t,\boldsymbol{\theta})}{\partial \boldsymbol{z}(t)}.
\end{equation}

Then from 

\begin{equation}\label{g5}
\boldsymbol{B}(t)\boldsymbol{A}(t)=I,
\end{equation}
we have
\begin{equation}\label{g6}
-\boldsymbol{B}(t)\frac{d\boldsymbol{A}(t)}{dt}=\frac{d\boldsymbol{B}(t)}{dt}\boldsymbol{A}(t), 
\end{equation}
thus, using this expression in Equation \ref{g4} we get

\begin{equation}\label{g7}
\frac{d\boldsymbol{B}(t)}{dt}\boldsymbol{A}(t)=\frac{\partial f(\boldsymbol{z}(t),t,\boldsymbol{\theta})}{\partial \boldsymbol{z}(t)}.
\end{equation}
Finally, we get the desired result by multiplying both sides from the right with $\boldsymbol{B}(t)$ .

\section{Equivalence Between Continuous Total Derivative Decomposition and the Continuous Backpropagation}\label{appendixH}

In Section \ref{adsens}. we derived the expression of continuous backpropagation from  the continuous total derivative decomposition. However, we can also derive the the formula of the continuous total derivative decomposition (Equation \ref{conttotderiv1}) from continuous backpropagation (Equation \ref{conttotderiv2}). Indeed, for a function $f=f(\boldsymbol{z}(t),\boldsymbol{\theta}(t),t)$, where $\boldsymbol{\theta}(t)=g(\boldsymbol{\theta},t)$, we can see $f$ as $h(\boldsymbol{z}(t),\boldsymbol{\theta},t)=f(\boldsymbol{z}(t),\boldsymbol{\theta}(t),t)$. Hence,
\begin{equation}\label{h1}
\frac{\partial L}{\partial \boldsymbol{\theta}}= \int_0^T \frac{\partial L}{\partial \boldsymbol{z}(t)} \frac{d h(\boldsymbol{z}(t),\boldsymbol{\theta},t)}{d \boldsymbol{\theta}}dt=\int_0^T \frac{\partial L}{\partial \boldsymbol{z}(t)} \frac{d f(\boldsymbol{z}(t),\boldsymbol{\theta}(t),t)}{d \boldsymbol{\theta}}dt,
\end{equation}
therefore
\begin{equation}\label{h2}
\frac{\partial L}{\partial \boldsymbol{\theta}}=\int_0^T \frac{\partial L}{\partial \boldsymbol{z}(t)} \frac{\partial f(\boldsymbol{z}(t),\boldsymbol{\theta}(t),t)}{\partial \boldsymbol{\theta}(t)}\frac{d\boldsymbol{\theta}(t)}{d \boldsymbol{\theta}}dt.
\end{equation}
Setting $L= \boldsymbol{z}(T)$, gives the desired result.

\section{Additional Details About the Experiments}\label{appendixI}

As mentioned in the main paper, we optimized the architecture of FFJORD first and fine-tuned its hyper-parameters.

This architecture remains unchanged in the AFFJORD model for fair comparison, and we only fine-tuned the augmented structure and dimension in addition. Indeed, as it can be seen, the results we report for FFJORD (0.96 MNIST, 3.37 CIFAR) are better than those reported on the original paper (0.99 MNIST, 3.40 CIFAR). The aforementioned improvements were a result of reducing the number of parameters during our tuning process, by reducing the number of CNF blocks from 2 to 1. The total number of parameters for different architectures of FFJORD and AFFJORD can be found on Table \ref{parameter-table}.
It should be emphasized that in the hypernet architecture, the parameters of AFFJORD are contained inside the parameters in the main architecture of FFJORD, so a bigger model is not being used, simply the flexibility of evolution of the main parameters in time is being increased.

 \begin{table}
\centering

  \caption{Number of Total Parameters of FFJORD and AFFJORD (Concat and Hypernet Architecture) for MNIST and CIFAR-10. The Multiscale Architecture Is Used in All Cases. \\ }

  \label{parameter-table}
  \begin{tabular}{l|ll|ll}
    \multirow{2}{*}{Dataset} &
      \multicolumn{2}{c|}{Concat} &
      \multicolumn{2}{c}{Hypernet} \\
    & FFJORD & AFFJORD & FFJORD & AFFJORD \\
   \hline
    MNIST & 400323 & 417629 & 783019 & 4556169  \\
    
    CIFAR10 & 679441 & 820815 & 1332361 & 8976115  \\
    CelebA ($32\times32$) & 679441 & 820815 & 1332361 & 8976115  \\
    \hline

  \end{tabular}
\end{table}

More generally, our experiments show that augmented architectures generally improve the performance of the standard non-augmented FFJORD counterparts, independently from the baseline architecture used. Furthermore, the farther the performance of FFJORD is from being optimal, the greater are the improvements when using AFFJORD. 

\end{document}